\PassOptionsToPackage{table}{xcolor}
\documentclass{article}
\usepackage{iclr2025_conference,times}

\usepackage{amsmath,amsfonts,bm}

\def\eqref#1{equation~\ref{#1}}

\def\1{\bm{1}}

\DeclareMathAlphabet{\mathsfit}{\encodingdefault}{\sfdefault}{m}{sl}
\SetMathAlphabet{\mathsfit}{bold}{\encodingdefault}{\sfdefault}{bx}{n}

\usepackage{hyperref}
\hypersetup{hidelinks}
\usepackage{xurl}
\usepackage{booktabs}
\usepackage{multirow}
\usepackage{xcolor}
\usepackage{amssymb}
\usepackage{amsthm}
\usepackage{txfonts}
\usepackage{graphicx}
\usepackage{subcaption}
\usepackage{wrapfig}
\newtheorem{theorem}{Theorem}
\usepackage{threeparttable}
\usepackage{enumitem}
\usepackage{array}
\usepackage{iftex}
\usepackage{fancyhdr}
\iclrfinalcopy
\ifPDFTeX
  \newcommand{\pTwoSArtifactFont}{\ttfamily\small}
  \newcommand{\pTwoSArtifactSetSymbols}{}
  \DeclareUnicodeCharacter{2705}{\ensuremath{\checkmark}} %
  \DeclareUnicodeCharacter{2224}{\ensuremath{\nmid}}     %
\else
  \usepackage{fontspec}
  \newfontfamily\pTwoSArtifactFont{DejaVuSansMono.ttf}[Scale=0.88]
  \newfontfamily\pTwoSArtifactSymbolFont{FreeSerif.otf}
  \begingroup
  \catcode`✅=\active
  \catcode`∤=\active
  \gdef\pTwoSArtifactSymbolDefinitions{%
    \def✅{{\pTwoSArtifactSymbolFont\char"2705}}%
    \def∤{{\pTwoSArtifactSymbolFont\char"2224}}%
  }
  \endgroup
  \newcommand{\pTwoSArtifactSetSymbols}{%
    \fvset{codes={\catcode`✅=\active\catcode`∤=\active},%
      defineactive=\pTwoSArtifactSymbolDefinitions}%
  }
\fi
\usepackage{tcolorbox}
\tcbuselibrary{breakable,skins}
\usepackage{fvextra}
\definecolor{PTwoSSkillBlue}{HTML}{315C82}
\definecolor{PTwoSSkillGreen}{HTML}{237B57}
\newtcolorbox{PTwoSSkill}[2]{%
  enhanced,breakable,
  colback=#1!3!white,colframe=#1,colbacktitle=#1,coltitle=white,
  title={#2},title after break={#2\ (continued)},
  fonttitle=\sffamily\bfseries\small,
  boxrule=0.5pt,arc=1pt,
  left=6pt,right=6pt,top=5pt,bottom=5pt,
  before skip=9pt,after skip=10pt,pad at break=2mm,
  before upper={\pTwoSArtifactSetSymbols
    \fvset{formatcom=\pTwoSArtifactFont,
      breaksymbolleft={},breaksymbolright={}}}%
}

\usepackage{amsmath,amssymb}
\usepackage{booktabs,array,tabularx,longtable}
\usepackage{float}
\usepackage{iftex}
\ifPDFTeX
  \usepackage[T1]{fontenc}
  \usepackage[utf8]{inputenc}
  \newcommand{\pTwoSPromptFont}{\ttfamily}
\else
  \usepackage{fontspec}
  \newfontfamily\pTwoSPromptFont{DejaVuSansMono.ttf}[Scale=0.88]
\fi
\usepackage{tcolorbox}
\tcbuselibrary{breakable,skins}
\usepackage{fvextra}
\usepackage{hyperref}

\definecolor{PTwoSPromptBlue}{HTML}{315C82}
\newtcolorbox{PTwoSPrompt}[1]{%
  enhanced,breakable,
  colback=PTwoSPromptBlue!3!white,colframe=PTwoSPromptBlue,
  colbacktitle=PTwoSPromptBlue,coltitle=white,
  title={#1},title after break={#1\ (continued)},
  fonttitle=\sffamily\bfseries\small,
  boxrule=0.5pt,arc=1pt,
  left=6pt,right=6pt,top=5pt,bottom=5pt,
  before skip=9pt,after skip=10pt,pad at break=2mm,
  before upper={\fvset{formatcom=\pTwoSPromptFont,
    breaksymbolleft={},breaksymbolright={}}}%
}

\makeatletter
\@ifundefined{algorithm}{%
  \floatstyle{ruled}
  \newfloat{algorithm}{tbp}{loa}
  \floatname{algorithm}{Algorithm}
  \floatstyle{plain}
}{}
\makeatother

\newcommand{\method}{\textsc{Prompt2Skill}\xspace}

\title{Prompt2Skill: Unsupervised Skill Optimization From Natural Language Instructions}

\author{Bo Ni$^{1}$, Li Li$^{2}$, Ryan A. Rossi$^{3}$, Franck Dernoncourt$^{3}$, Tyler Derr$^{1}$ \\
$^{1}$Vanderbilt University \quad $^{2}$University of Southern California \quad $^{3}$Adobe Systems \\
\texttt{\{bo.ni,tyler.derr\}@vanderbilt.edu} \quad \texttt{lli@usc.edu} \\
\texttt{\{ryrossi,dernonco\}@adobe.com}
}

\usepackage{xspace}

\newtcolorbox{abstractbox}{
  enhanced,
  colback=blue!3,     %
  colframe=gray!40,   %
  boxrule=0.4pt,      %
  arc=2.5mm,          %
  left=5mm, right=5mm, top=4mm, bottom=4mm,
  width=\textwidth
}

\begin{document}
\maketitle
\lhead{Preprint. Under review.}

\begin{abstractbox}
\noindent
Skills are external artifacts that Large Language Models (LLMs) consume at inference time to improve their performance on specialized domains by incorporating relevant procedural and domain knowledge. Expert-authored skills are expensive to produce, and the resulting artifacts are not optimized for the specific model that consumes them, whose failure modes can vary with version, scale and training. In addition, emerging tasks may fall outside the scope of existing skill libraries, creating a need to develop new skills before curated training data become available. Recent works have explored automated skill optimization through reflection, but they require a curated, in-distribution training set, which users might not always have. To address these limitations, we present \textbf{Prompt2Skill}, a framework that builds skills from natural-language task description alone. From the prompt, the system derives a task specification, discovers or synthesizes datasets, and refines the skill in a closed loop of reflective editing. Across four domains spanning question answering, reading comprehension, spreadsheet manipulation, and mathematical reasoning, Prompt2Skill consistently outperforms the direct prompting baseline, achieving an average improvement of \textbf{10.8\%} across open-source and frontier models.
\medskip
\centerline{\textbf{Code:} \url{https://github.com/Arstanley/Prompt2Skill}}
\end{abstractbox}
\begin{figure}[h]
    \centering
    \vskip -0.5ex
    \includegraphics[width=\linewidth]{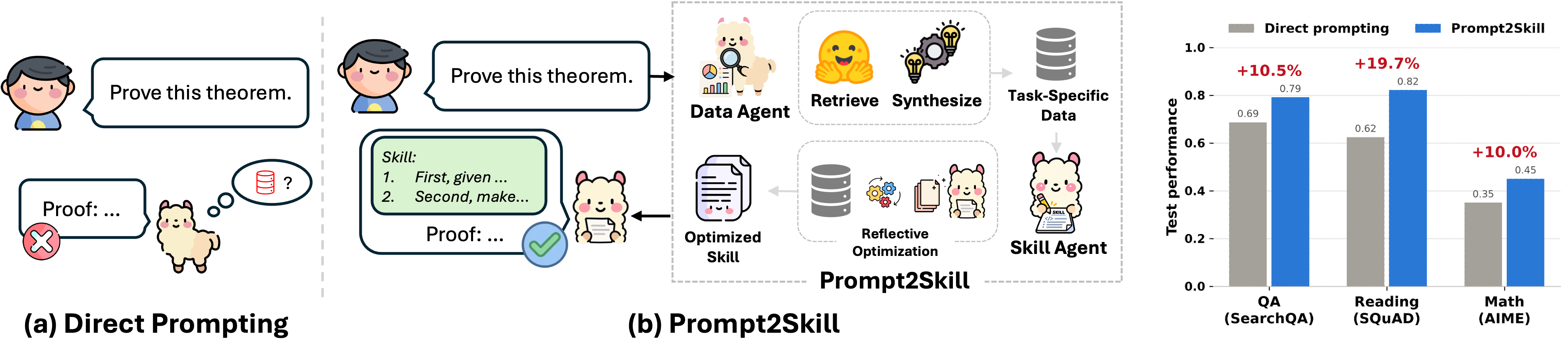}
    \vskip -1.5ex
    \caption{Prompt2Skill vs. Direct Prompting.}
    \vskip -2ex
    \label{fig:placeholder}
\end{figure}
\section{Introduction}
\vspace{-0.5ex}
Large Language Models are increasingly deployed as models equipped with skills, which are external artifacts that the model consumes at inference time to improve its performance on specialized tasks~\citep{anthropic2025agentskills}. Often, skills are human-readable markdown files placed into the model's context, equipping large language models with specific procedures, domain conventions, and tool-use knowledge that they do not reliably exhibit parametrically, and through which a general-purpose model can improve performance on a specialized task without any weight updates~\citep{xu2026agentskillslargelanguage,li2026skillsbenchbenchmarkingagentskills}. 

Today, skills are predominantly authored by human experts~\citep{li2026skillsbenchbenchmarkingagentskills, anthropic2025agentskills}, which limits scalability: each new task demands domain knowledge, familiarity with the skill format, and iteration against model behavior~\citep{ni2026trace2skilldistilltrajectorylocallessons}. To address this problem, recent work explores automated skill creation from agent experience. For example, Trace2Skill~\citep{ni2026trace2skilldistilltrajectorylocallessons} consolidates execution trajectories in parallel into a unified skill directory via inductive reasoning, compressing recurring failures and workarounds into standard operating procedures. SkillOpt~\citep{yang2026skilloptexecutivestrategyselfevolving}, on the other hand, casts skill creation as controllable text-space optimization: a separate optimizer model turns scored rollouts into bounded add/delete/replace edits on the skill document, accepting an edit only when it strictly improves a held-out validation score. However, these optimizers presuppose a curated, in-distribution set of training tasks with reliable labels — the very resource a deployed user lacks. Prior work shows that shrinking the training set to a single example collapses its gains to the level of a data-blind draft~\citep{yang2026skilloptexecutivestrategyselfevolving}, limiting a more realistic setting where a user who has only a task in mind, and perhaps a couple of examples: "I want a model that reads a wikipedia passage and answers questions about it."

To this end, we present \method, a multi-agent framework that builds an optimized skill from a natural-language task description alone. From the prompt, the Data Agent derives a task specification, including the input–output contract and answer format, and retrieves candidate datasets from online resources. To adapt to a wide range of requests, the agent then optionally processes or synthesizes data to fit the task: reformatting retrieved items into the task's interface, or generating grounded items when none does. The Skill Agent then refines the skill in a closed loop of rollouts and reflective editing: a reflector model proposes candidate edits from the target model's own failures, and each edit is accepted only if it wins a statistically significant paired comparison on a fresh batch of held-out items, with a frozen validation set consulted exactly once for final selection. 

We conduct extensive experiments on four domains, question answering, reading comprehension, spreadsheet manipulation, and mathematical reasoning, with models spanning open-source and commercial systems. Across configurations, \method improves performance by an average of 10.8\% points, on average, and never significantly regresses. We also compare against skills authored by a frontier LLM from the task description alone, demonstrating the effectiveness of incorporating data retrieval and optimization for skill authorization. In summary, our contributions can be summarized as follows:

\begin{itemize}
    \vspace{-1ex}
    \item \textbf{A prompt-to-skill framework.} We present \method, to our knowledge the first end-to-end multi-agent framework that turns a bare natural-language task description into a model-optimized skill, via agentic data discovery and synthesis followed by closed-loop refinement on the skill, with no weight updates and no user-supplied data.
    \vspace{-0.5ex}
    \item \textbf{A comprehensive empirical study.} Across four domains and five open-source and commercial target models, \method improves over direct prompting by 10.8\% points on average and never significantly regresses. Extensive ablations and case studies further isolate the contribution of each component.
\end{itemize}

\vspace{-1ex}
\section{Problem Definition}
\vspace{-0.5ex}
\label{sec:problem}
\paragraph{Skills.}
Let $M$ denote a Large Language Model with a finite token vocabulary $\Sigma$.
A \emph{skill} is a text artifact $s \in \Sigma^{*}$ that is placed in $M$'s
context at inference time, where $\Sigma^{*}$ is the set of all finite
sequences of tokens drawn from $\Sigma$. We write $M_{s}$ for the resulting
skill-equipped model, and $M_{\emptyset}$ for the model prompted
directly. A task is characterized by a distribution
$\mathcal{D}$ over input--output pairs together with an evaluation metric
$\mu$; the value of a skill is
\begin{equation}
    J(s) \;=\; \mathbb{E}_{(x,y)\sim\mathcal{D}}
    \left[\, \mu\!\left(M_{s}(x),\, y\right) \right].
\label{eq:objective}
\end{equation}

\vspace{-1ex}
\paragraph{Prompt-to-Skill.}
Prior work on automated skill authoring and optimization assumes access to a
labeled training set $\{(x_i, y_i)\}$ drawn from $\mathcal{D}$, on which
candidate skills can be scored and selected. We remove this assumption and
study a more general problem in which $\mathcal{D}$ is never observed: the
system receives no samples from the task distribution, and the task is
specified only through a natural language description. Specifically, let
$d \in \Sigma^{*}$ be a natural language description of the target task, for
example \emph{``I want a model that reads a passage and answers questions
about it.''} The description $d$ implicitly fixes both the task distribution
$\mathcal{D}$ and the metric $\mu$ in Eq.~\eqref{eq:objective}, but the
system observes neither. Optionally, $d$ is accompanied by a small set of
examples $\mathcal{E}$ that illustrate the intended input and output format.
The supervised setting of prior work is the special case in which
$\mathcal{E}$ is a large labeled sample from $\mathcal{D}$; we focus on the
regime where $\mathcal{E}$ is empty or holds only a handful of examples, too
few to score and select candidate skills reliably. Formally, a method for this problem is a procedure
$\mathcal{A}: (d, \mathcal{E}, M, \mathcal{R}, \mathcal{G}, B) \mapsto \hat{s}$,
and its quality is $J(\hat{s})$, which can be measured only after
optimization, on a test set drawn from $\mathcal{D}$.

\section{\method}
\begin{figure}
    \centering
    \includegraphics[width=0.95\linewidth]{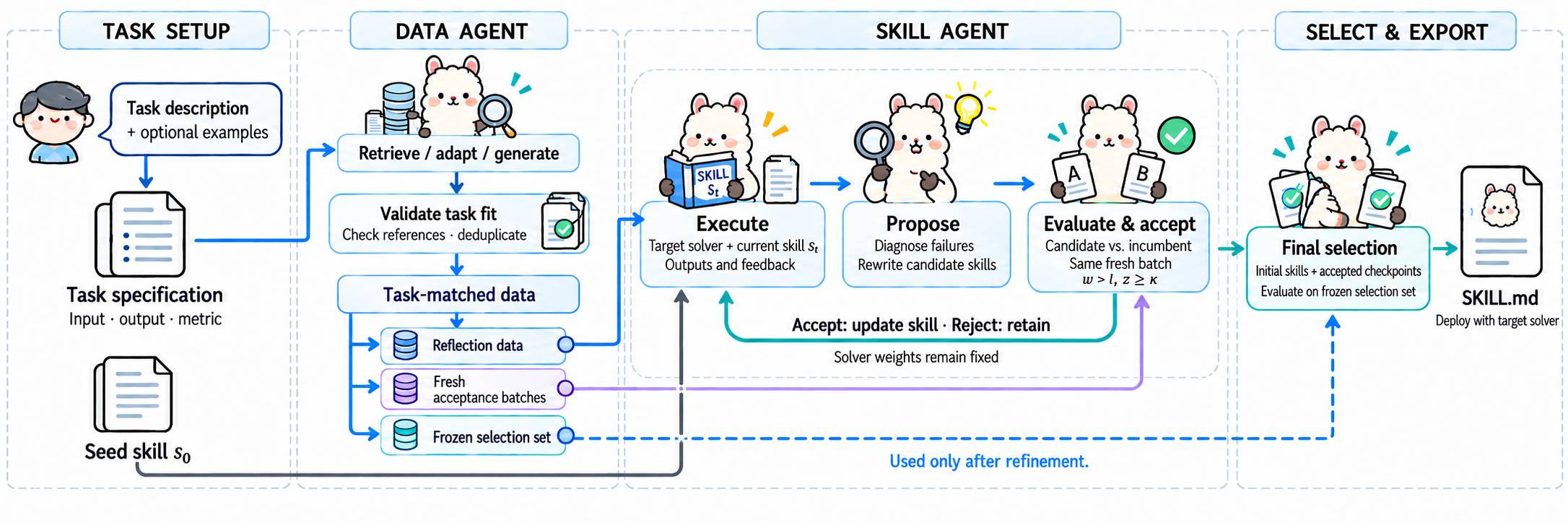}
    \caption{Overview of \method. The Data Agent turns the task description
    into a task specification, a seed skill, and a pool of validated items;
    the Skill Agent refines the seed in a closed loop of rollouts, reflective
    edits, and statistical acceptance on fresh batches. The final skill is
    chosen on a frozen validation set that is consulted exactly once.}
    \label{fig:framework}
\end{figure}

We present \method{}, an end-to-end framework that transforms a
natural-language task description and optional examples into a reusable
skill for a target model $M$. The framework consists of two agents.
The \textbf{Data Agent} translates the request into a task specification,
initializes a seed skill, and constructs a pool of validated items
through retrieval, processing, and grounded synthesis. The
\textbf{Skill Agent} refines the seed through repeated rollouts and
reflective edits, accepting a revision only when it passes a paired
statistical acceptance test on a fresh held-out batch. After refinement,
a frozen validation set is used in a single final selection stage.
The resulting skill is tailored to the specified task and target model
and can be deployed without updating the model's parameters.
Figure~\ref{fig:framework} summarizes the framework.
We describe task setup and the two agents below, and finally
Algorithm~\ref{alg:framework} recaps the complete procedure
\subsection{Task Setup}
\label{sec:task_setup}

The system receives a task description $d\in\Sigma^*$, optional examples
$\mathcal{E}$, and a target model $M$. Neither the target distribution
$\mathcal{D}$ nor its evaluation metric $\mu$ is directly available.
The setup stage therefore converts the request into an operational
task specification and an initial skill. Let $f_{\theta}(\cdot)$
denote an LLM-based operator conditioned on an instruction prompt
$\theta\in\Sigma^*$, with the underlying model parameters held fixed.
We jointly generate the specification and seed skill as
\begin{equation}
    (\tau,s_0)
    =
    f_{\theta_{\mathrm{setup}}}(d,\mathcal{E}),
    \qquad s_0\in\Sigma^*.
    \label{eq:task_setup}
\end{equation}
The prompt $\theta_{\mathrm{setup}}$ specifies the required output
schema and instructs the model to infer the task requirements from
$d$, when available, to clarify the intended
input and output format.

The specification $\tau$ records the input structure, task type,
language, expected answer style, and evaluation criterion.
Establishing these requirements before data acquisition provides
a fixed reference for assessing candidate sources and items.
The evaluation criterion determines an executable proxy metric
$\widehat{\mu}_{\tau}$, such as normalized exact matching, numeric
equivalence, or answer containment. We distinguish
$\widehat{\mu}_{\tau}$ from $\mu$: the former implements the system's
interpretation of the request, while the latter defines performance
on the target task.

The seed $s_0$ is a concise instruction that states the task and
required output format where the skill agent can optimize from. It is generated before inspecting acquired
data or observing model failures, so its content depends only on
the description and optional examples. The Data Agent uses $\tau$
to guide subsequent data acquisition and validation, while the
Skill Agent uses $s_0$ to initialize refinement.
The setup prompt $\theta_{\mathrm{setup}}$ remains fixed throughout
optimization; its complete instantiation, including the output
schema and optional example block, is provided in
Appendix~\ref{app:prompts}.

\subsection{Data Agent}
\label{sec:data_agent}

Given the task specification $\tau$ produced in the previous section, the data agent constructs a pool of input-reference pairs for skill optimization. Since the target distribution $\mathcal{D}$ is unobserved, data acquisition is guided by the requirement inferred from the requests. The agent will retrieve candidate sources, adapt their records to the required format, and generate additional items if needed. We discuss the agent in more details in the rest of this subsection.

\paragraph{Data Retrieval.}
Let $\mathcal{R}=\{r_1,\ldots,r_K\}$ denote the retrieval tools
available to the Data Agent. Each tool $r_k$ maps a query
$q\in\Sigma^*$ to a finite set of source identifiers, such as
dataset identifiers or article titles. Our implementation includes
interfaces to the Hugging Face dataset catalog and Wikipedia.
The framework specifies these interfaces, while the agent
determines which tools to invoke and what queries to issue.

Given the task specification $\tau$ and descriptions of the
available tools, the agent generates a set of retrieval requests:
\begin{equation}
    \mathcal{Q}_{\tau}
    =
    f_{\theta_{\mathrm{retrieve}}}(\tau,\mathcal{R})
    \subseteq \mathcal{R}\times\Sigma^*.
    \label{eq:retrieval_requests}
\end{equation}
Each pair $(r,q)\in\mathcal{Q}_{\tau}$ specifies a retrieval tool
and its query. Queries may express task categories or descriptive
keywords. To improve dataset discovery, the agent also generates
a hypothetical description of a suitable dataset and derives
search terms from it.

The requests are executed through the corresponding retrieval
interfaces, yielding the candidate source set
\begin{equation}
    \mathcal{C}_{\tau}
    =
    \bigcup_{(r,q)\in\mathcal{Q}_{\tau}} r(q).
    \label{eq:data_retrieval}
\end{equation}
Here, $f_{\theta_{\mathrm{retrieve}}}$ denotes LLM-based request
generation, whereas $r(q)$ denotes an external retrieval call.
The agent then inspects the returned sources, including their
schemas when available and samples of their content, to assess
their suitability for adaptation.

\paragraph{Data Adaptation and Generation.}
Retrieved sources may contain relevant information without directly
providing instances of the requested task. For example, if the user requests about a task on Wikipedia based question answering, the agent needs further data processing on the retrieved articles to fit the specific user requirements. They may use a different
input format, lack suitable reference answers, or omit required
artifacts. The Data Agent addresses these mismatches by adapting
existing records or generating new task instances from available
material.

For each source $b\in\mathcal{C}_{\tau}$, let $\mathcal{Z}_b$
denote its retrieved records and $h_b$ its available schema and
metadata. When the required information is present, the agent
inspects a sample
$\widetilde{\mathcal{Z}}_b\subseteq\mathcal{Z}_b$ and proposes
a conversion plan:
\begin{equation}
    \pi_b
    =
    f_{\theta_{\mathrm{adapt}}}
    \bigl(\tau,h_b,\widetilde{\mathcal{Z}}_b\bigr).
    \label{eq:adaptation_plan}
\end{equation}
The plan identifies the input and reference fields and specifies
how they should be assembled. An executable transformation
$T_{\pi_b}$ converts each usable record $z$ into an
input--reference pair $(x,y)$; records that cannot be converted
are discarded. For example, adapting a reading comprehension
dataset requires assembling the passage and question into $x$
and extracting the accepted answers as $y$. Here,
$f_{\theta_{\mathrm{adapt}}}$ uses an LLM to propose the plan,
while $T_{\pi_b}$ executes the specified conversion.

When structural conversion is insufficient, the agent generates
new items that satisfy the task requirements. For text tasks,
let $\mathcal{H}$ denote the retrieved material or examples
selected as generation context. Candidate items are produced as
\begin{equation}
    \mathcal{U}_{\mathrm{gen}}
    =
    f_{\theta_{\mathrm{gen}}}(\tau,\mathcal{H}).
    \label{eq:data_generation}
\end{equation}
Depending on the task, generation may derive questions and
references from retrieved passages or construct new self-contained
exercises using retrieved examples as a format reference.
When no suitable material is available, $\mathcal{H}$ is empty
and generation is guided by $\tau$. After adaptation or generation, all resulting items undergo validation before entering the optimization pool. The prompts $\theta_{\mathrm{adapt}}$ and
$\theta_{\mathrm{gen}}$ are provided in Appendix~\ref{app:prompts}.

\paragraph{Validation.}
Adapted and generated items are validated before entering the
optimization pool. Validation checks structural integrity, task
compatibility, and reference usability. Structural checks discard
items with empty inputs, missing references, or malformed fields.
For candidate datasets, the agent uses
$f_{\theta_{\mathrm{validate}}}$ to report the language, answer
style, task type, and input completeness of converted samples.
These observed properties are compared against the fixed
specification $\tau$. For example, a source containing questions
and answers is unsuitable for passage-based question answering
if its converted inputs omit the passages.

The automatic text pipeline additionally applies an answerability
screen using the fixed task-derived seed $s_0$. Let
$\mathcal{B}^{\mathrm{probe}}_b$ denote a small sample from source
$b$. Using the binary correctness evaluator
$\widehat{\mu}_{\tau}$, the source passes this screen when
\begin{equation}
    \sum_{(x,y)\in\mathcal{B}^{\mathrm{probe}}_b}
    \widehat{\mu}_{\tau}\!\left(M_{s_0}(x),y\right)
    > 0.
    \label{eq:source_screening}
\end{equation}
Sources on which the seed-equipped model fails every sampled
item are rejected. This heuristic screens for unusable references,
input mismatches, and tasks beyond the model's demonstrated
capabilities on the sample.

Synthesized text items undergo an additional consistency check.
An item $(x,y)$ passes when either the seed-equipped model or a
directly prompted model produces an accepted answer:
\begin{equation}
    \max_{s\in\{s_0,\emptyset\}}
    \widehat{\mu}_{\tau}\!\left(M_s(x),y\right)
    = 1.
    \label{eq:synthetic_screening}
\end{equation}

We write $a_{\tau}(x,y)=1$ when an item passes the applicable
source and item checks, and $a_{\tau}(x,y)=0$ otherwise.
These checks provide practical evidence of usability; solver
agreement alone does not establish reference correctness or
representativeness of the target distribution.
\paragraph{Pool Construction.}
Let $\mathcal{U}$ denote the candidate items obtained through
adaptation and generation, and let $a_{\tau}(x,y)\in\{0,1\}$
indicate whether an item passes the applicable source and item
checks. We summarize validation and duplicate filtering as
\begin{equation}
    \mathcal{P}
    =
    \operatorname{Dedup}
    \left(
        \left\{(x,y)\in\mathcal{U}:
        a_{\tau}(x,y)=1\right\}
    \right),
    \label{eq:validated_pool}
\end{equation}
where $\operatorname{Dedup}$ denotes deterministic filtering
based on the pipeline's item keys: normalized input keys for
text items and item identifiers for spreadsheet tasks.
The pool can be extended during refinement as additional
items are retrieved or generated.

The data serve three roles. At round $t$, the reflection set
$\mathcal{T}_t$ supplies model rollouts and failure examples
for proposing skill edits. A fresh acceptance batch
$\mathcal{B}_t$ supports comparisons between those proposals
and the incumbent skill. A frozen selection set $\mathcal{V}$
is established at initialization and retains the same membership
throughout optimization. Within each round,
\begin{equation}
    \mathcal{T}_t,\mathcal{B}_t
    \subseteq \mathcal{P}\setminus\mathcal{V},
    \qquad
    \mathcal{T}_t\cap\mathcal{B}_t=\varnothing.
    \label{eq:data_roles}
\end{equation}

For any nonempty batch $\mathcal{B}\subseteq\mathcal{P}$,
the empirical skill score is
\begin{equation}
    \widehat{J}_{\mathcal{B}}(s)
    =
    \frac{1}{|\mathcal{B}|}
    \sum_{(x,y)\in\mathcal{B}}
    \widehat{\mu}_{\tau}\!\left(M_s(x),y\right).
    \label{eq:proxy_objective}
\end{equation}
Acceptance compares the incumbent and proposed skills on the
same $\mathcal{B}_t$, evaluating edits on items beyond those
used to construct them. These scores guide optimization on the constructed proxy data.
Performance on the target distribution $\mathcal{D}$ is measured
afterward on the withheld target test set. The data acquisition
and validation prompts are provided in Appendix~\ref{app:prompts}.

\subsection{Skill Agent}
\label{sec:skill_agent}

Given the items pool $\mathcal{P}$, The Skill Agent refines an initial skill using reflective edits, and paired evaluation on fresh acceptance
batches. Each round produces candidate skills from observed failures,
evaluates them against the incumbent, and retains an update only
when it satisfies the acceptance criterion. After refinement,
a comparison on the frozen selection set determines
the exported skill.

\paragraph{Rollouts and Executions.}
Let $s_t$ denote the incumbent skill at round $t$, starting with
$s_1=s_0$ under task-derived initialization. Supplied or
data-informed initial skills can replace this starting incumbent.
For each item $(x_i,y_i)\in\mathcal{T}_t$, the agent executes
the skill-equipped model and evaluates the resulting output:
\begin{equation}
    o_{t,i}=M_{s_t}(x_i),
    \qquad
    r_{t,i}
    =
    \widehat{\mu}_{\tau}(o_{t,i},y_i).
    \label{eq:skill_rollouts}
\end{equation}
The implemented evaluators return binary outcomes
$r_{t,i}\in\{0,1\}$, whose average gives
$\widehat{J}_{\mathcal{T}_t}(s_t)$.

\paragraph{Reflective Editing.}
Following prior work~\citep{agrawal2026gepareflectivepromptevolution, yang2026skilloptexecutivestrategyselfevolving},
the Skill Agent uses rollout feedback to identify failure
patterns and propose revisions to the incumbent skill.
Let $\mathcal{F}_t$ denote the available feedback from
unsuccessful executions. This feedback can contain
input excerpts, reference answers, and the model's incorrect
answers.

For a sampled subset
$\mathcal{F}_{t,j}\subseteq\mathcal{F}_t$, the agent proposes
a candidate skill as
\begin{equation}
    s'_{t,j}
    =
    f_{\theta_{\mathrm{edit}}}
    \left(
        \tau,s_t,
        \widehat{J}_{\mathcal{T}_t}(s_t),
        \mathcal{F}_{t,j}
    \right).
    \label{eq:skill_edit}
\end{equation}

The resulting candidates form
$\mathcal{S}_t=\{s'_{t,1},\ldots,s'_{t,K_t}\}$,
which are passed to the acceptance step below.
The reflection prompts and authoring procedures are provided
in Appendix~\ref{app:prompts}.

\paragraph{Evaluation and Acceptance.}
The agent evaluates each candidate $s'_{t,j}$ and the incumbent
$s_t$ on the same fresh batch $\mathcal{B}_t$. Let $w_{t,j}$
count items solved only by the candidate and $\ell_{t,j}$
those solved only by the incumbent. A candidate qualifies when
$w_{t,j}>\ell_{t,j}$ and $z_{t,j}\geq\kappa$, where
\begin{equation}
    z_{t,j}
    =
    \begin{cases}
        \displaystyle
        \frac{|w_{t,j}-\ell_{t,j}|-1}
             {\sqrt{w_{t,j}+\ell_{t,j}}},
        & w_{t,j}+\ell_{t,j}>0,\\[6pt]
        0, & \text{otherwise}.
    \end{cases}
    \label{eq:acceptance_statistic}
\end{equation}
We use $\kappa=1$ as a per-candidate screening threshold.
Among qualifying candidates, the agent accepts the one with
the largest net gain $w_{t,j}-\ell_{t,j}$; if none qualifies,
it retains $s_t$. Accepted skills become the next incumbent
and are saved as checkpoints for final selection. The round
budget and stopping criteria are specified in Appendix~\ref{app:implementation-details}.

\subsection{Selection and Export}
\label{sec:selection}

Whenever a candidate is accepted and becomes the incumbent,
\method{} saves its skill text. Let $\mathcal{S}_{\mathrm{acc}}$
denote these previously accepted skills and
$\mathcal{S}_{\mathrm{init}}$ the initial skills retained as
fallback options. After refinement, all eligible skills are
evaluated on the frozen selection set $\mathcal{V}$.
For nonempty $\mathcal{V}$, the selected skill is
\vspace{-0.75ex}
\begin{equation}
\hat{s}
\in
\arg\max_{
s\in\mathcal{S}_{\mathrm{init}}
\cup\mathcal{S}_{\mathrm{acc}}
}
\widehat{J}_{\mathcal{V}}(s).
\label{eq:final_selection}
\end{equation}
The common selection set makes skills accepted on different
batches comparable. Selection scores are not fed into subsequent
refinement. The selected skill is exported as \texttt{SKILL.md}
and supplied to the target model at inference time.

Intuitively, terminal selection can recover an earlier improvement
when later edits generalize less well, while retaining initial
skills provides fallback options. Under an independent selection
sample and bounded proxy mismatch, we bound the performance gap
between the exported skill and the best retained skill.
Appendix~\ref{app:selection_guarantee} formalizes this guarantee
and its implications for the framework.

\vspace{-0.5ex}
\subsection{Algorithm}
\vspace{-0.5ex}

We present the \method{} algorithm in Algorithm~\ref{alg:framework}.
The procedure constructs task-matched proxy data and refines an
initial skill through execution feedback, reflective editing,
and paired acceptance on fresh batches.
After refinement, it compares the retained skills on a frozen
selection set and exports the selected artifact for inference.

\begin{algorithm}[t]
\scriptsize 
\caption{\method construction and refinement.}
\label{alg:framework}
\textbf{Input:} Task description $d$, optional examples $\mathcal E$,
target model $M$, acquisition tools, and refinement budget.\\
\textbf{Output:} A skill artifact $\hat s$.
\begin{enumerate}[leftmargin=6ex]
\setlength{\itemsep}{2pt}
\setlength{\parskip}{0pt}
\item Establish the task specification $\tau$ and proxy evaluator
$\widehat\mu_\tau$. Choose an initial incumbent $s$ and retain the
eligible initial skills $\mathcal S_{\mathrm{init}}$.
\item Retrieve or generate task items, adapt their representations, and
apply the applicable validation and duplicate checks. Establish the
frozen selection set $\mathcal V$ and the initial reflection pool.
Set $\mathcal S_{\mathrm{acc}}\leftarrow\varnothing$.
\item For each permitted refinement round $t$:
  \begin{enumerate}
  \setlength{\itemsep}{1pt}
  \item Acquire additional reflection items as required by the task driver
  and form $\mathcal T_t$, excluding selection items.
  \item Execute $M_s$ on $\mathcal T_t$ and compute its proxy score and
  failure feedback $\mathcal F_t$.
  \item Propose the ordered candidates
  $\mathcal S_t=(s'_{t,1},\ldots,s'_{t,K_t})$ using reflective editing.
  \item Acquire a fresh acceptance batch $\mathcal B_t$, excluding
  current reflection and selection items under the driver's item keys.
  Evaluate the incumbent once on $\mathcal B_t$.
  \item Evaluate every candidate on the same $\mathcal B_t$.
  Compute paired wins $w_{t,j}$, losses $\ell_{t,j}$, and the statistic
  $z_{t,j}$ in Eq.~\eqref{eq:appendix_acceptance}.
  \item Among candidates with $w_{t,j}>\ell_{t,j}$ and $z_{t,j}\geq1$,
  select the first candidate attaining the largest net gain
  $w_{t,j}-\ell_{t,j}$. If one exists, replace $s$ with its text and
  retain that text in $\mathcal S_{\mathrm{acc}}$.
  \item Record the round outcome and apply the driver's stopping rule.
  \end{enumerate}
\item When $\mathcal V$ is nonempty, evaluate the eligible initial skills,
the incumbent, and previously accepted skills on $\mathcal V$.
Select the highest-scoring entry using the driver's tie rule.
\item Export the selected text as \texttt{SKILL.md}.
\end{enumerate}
\end{algorithm}

\section{Experiment}
In this section, we evaluate whether \method{} can construct useful
skills from task descriptions across different tasks and target models,
following the prompt-to-skill setting in Section~\ref{sec:problem}.
Our experiments compare the resulting skills with direct prompting
and fixed skill baselines, and examine sensitivity to initialization
and optional input--output examples.
We also explore the application of \method{} to NLP AutoML,
using skill construction as an alternative to task-specific
fine-tuning. Preliminary experiments and analysis are provided
in Appendix~\ref{sec:nlp_automl}.

\subsection{Experimental Setup}
\label{sec:experimental_setup}

\paragraph{Task Prompts.}
We evaluate \method{} in the prompt-to-skill setting:
each task is specified through a natural-language description
$d$ of the desired behavior.
The main experiments use no user-provided examples.
For example, the reading-comprehension task is specified as
\emph{``I want a model that reads a passage and answers
questions about it.''}
This description guides task setup, proxy-data acquisition,
and skill refinement.
The resulting skill is then evaluated on the corresponding
held-out benchmark.
The exact task descriptions and the prompt templates used
by the framework are provided in
Appendix~\ref{app:prompts}.

\paragraph{Benchmarks.}
We evaluate four tasks spanning factual question answering
(SearchQA~\citep{dunn2017searchqanewqadataset}),
reading comprehension
(SQuAD~\citep{rajpurkar2016squad100000questionsmachine}),
mathematical reasoning (AIME~\citep{aime25}), and
spreadsheet manipulation
(SpreadsheetBench~\citep{ma2024spreadsheetbenchchallengingrealworld}).
These benchmarks measure the performance of the constructed
skills on their intended tasks.
Dataset splits, evaluation metrics, and execution settings
are provided in Appendix~\ref{app:datasets}.

\paragraph{Target Models.}
We consider three open-weight solvers:
Qwen3-8B, Qwen3-32B, and Llama-3.2-1B-Instruct;
and two commercial solvers:
Claude Haiku 4.5 and GPT-5.5.
The target model executes skills and supplies the outcomes
used during refinement, while GPT-5.5 provides reflective
feedback for the optimization runs.

\paragraph{Baselines.}
We compare \method against four baselines.
\textit{Direct} supplies the task instruction and required
output format without an additional skill.
\textit{Off-the-shelf} supplies an externally published skill
selected for domain relevance.
\textit{LLM-generated} uses a skill authored once by
Qwen3-32B from a domain description, without execution
feedback or iterative refinement. Details on the off-the-shelf skills and LLM-generated skills are provided in Appendix~\ref{app:off_the_shelf_skills}, and baselines details are provided in Appendix~\ref{app:baselines}.

\providecommand{\pmv}[2]{%
  \mbox{#1\,{\scriptsize$\pm$\,#2}}%
}

\providecommand{\modellabel}[1]{%
  \rotatebox[origin=c]{90}{%
    \scriptsize #1%
  }%
}

\providecommand{\splitmodellabel}[2]{%
  \rotatebox[origin=c]{90}{%
    \shortstack[c]{%
      \scriptsize #1\\[-1pt]
      \scriptsize #2%
    }%
  }%
}

\providecommand{\avgain}[2]{\cellcolor{green!#1}#2}
\providecommand{\avgloss}[2]{\cellcolor{red!#1}#2}

\begin{table*}[t]
\centering
\scriptsize  
\setlength{\tabcolsep}{3.8pt}
\renewcommand{\arraystretch}{1.10}

\caption{
Main results comparing direct prompting, existing general-purpose skills
(\textit{Off-the-shelf}), skills generated directly by an LLM
(\textit{LLM-generated}), and \method{} (P2S).
Llama-3.2-1B was not able to complete any SpreadsheetBench task.
}
\label{tab:main_results}
\vspace{-0.75ex}

\begin{threeparttable}

\begin{tabular}{@{}
    >{\centering\arraybackslash}p{1.35em}
    >{\centering\arraybackslash}p{2.3em}
    l
    cccc
    c
@{}}

\toprule
& & Method
& SearchQA
& SQuAD
& AIME
& Spreadsheet
& Avg.\ $\Delta$ \\
\midrule

\multirow[c]{12}{*}{%
  \rotatebox[origin=c]{90}{\textsc{Open-source}}%
}

&
\multirow[c]{4}{*}{\splitmodellabel{Qwen3}{8B}}
& Direct
& \pmv{0.431}{0.002}
& \pmv{\underline{0.738}}{0.001}
& \pmv{\textbf{0.208}}{0.014}
& \pmv{0.044}{0.005}
& 0.0\% \\

&
& Off-the-shelf
& \pmv{0.496}{0.003}
& \pmv{0.657}{0.004}
& \pmv{0.128}{0.013}
& \pmv{\textbf{0.085}}{0.007}
& \avgain{15}{+14.7\%} \\

&
& LLM-generated
& \pmv{\underline{0.557}}{0.003}
& \pmv{0.611}{0.001}
& \pmv{0.158}{0.017}
& \pmv{\underline{0.082}}{0.011}
& \avgain{19}{+18.6\%} \\

&
& P2S
& \pmv{\textbf{0.640}}{0.019}
& \pmv{\textbf{0.768}}{0.097}
& \pmv{\underline{0.175}}{0.000}
& \pmv{0.079}{0.011}
& \avgain{29}{\textbf{+29.1\%}} \\

\cmidrule(lr){2-8}

&
\multirow[c]{4}{*}{\splitmodellabel{Qwen3}{32B}}
& Direct
& \pmv{\underline{0.613}}{0.002}
& \pmv{\underline{0.694}}{0.002}
& \pmv{\underline{0.211}}{0.017}
& \pmv{0.152}{0.023}
& 0.0\% \\

&
& Off-the-shelf
& \pmv{0.537}{0.005}
& \pmv{0.630}{0.003}
& \pmv{0.136}{0.013}
& \pmv{\underline{0.180}}{0.012}
& \avgloss{10}{-9.7\%} \\

&
& LLM-generated
& \pmv{0.608}{0.002}
& \pmv{0.632}{0.002}
& \pmv{0.178}{0.013}
& \pmv{0.179}{0.026}
& \avgloss{2}{-1.9\%} \\

&
& P2S
& \pmv{\textbf{0.778}}{0.019}
& \pmv{\textbf{0.825}}{0.009}
& \pmv{\textbf{0.222}}{0.017}
& \pmv{\textbf{0.233}}{0.019}
& \avgain{26}{\textbf{+26.1\%}} \\

\cmidrule(lr){2-8}

&
\multirow[c]{4}{*}{\splitmodellabel{Llama-3.2}{1B}}
& Direct
& \pmv{0.140}{0.003}
& \pmv{\underline{0.207}}{0.002}
& \pmv{\underline{0.017}}{0.000}
& N/A
& 0.0\%\tnote{c} \\

&
& Off-the-shelf
& \pmv{0.120}{0.004}
& \pmv{0.048}{0.001}
& \pmv{\underline{0.017}}{0.000}
& N/A
& \avgloss{30}{-30.4\%} \\

&
& LLM-generated
& \pmv{\underline{0.273}}{0.007}
& \pmv{0.090}{0.001}
& \pmv{\textbf{0.031}}{0.005}
& N/A
& \avgain{40}{+40.3\%} \\

&
& P2S
& \pmv{\textbf{0.340}}{0.064}
& \pmv{\textbf{0.492}}{0.048}
& \pmv{\underline{0.017}}{0.014}
& N/A
& \avgain{60}{\textbf{+93.5\%}} \\

\midrule

\multirow[c]{8}{*}{%
  \rotatebox[origin=c]{90}{\textsc{Commercial}}%
}

&
\multirow[c]{4}{*}{\splitmodellabel{Claude}{Haiku 4.5}}
& Direct
& \pmv{\underline{0.857}}{0.000}
& \pmv{0.416}{0.000}
& \pmv{0.467}{0.014}
& \pmv{\underline{0.297}}{0.010}
& 0.0\% \\

&
& Off-the-shelf
& \pmv{\textbf{0.891}}{0.000}
& \pmv{0.294}{0.001}
& \pmv{\underline{0.511}}{0.017}
& \pmv{0.271}{0.014}
& \avgloss{6}{-6.2\%} \\

&
& LLM-generated
& \pmv{0.854}{0.000}
& \pmv{\underline{0.499}}{0.000}
& \pmv{\textbf{0.517}}{0.008}
& \pmv{\textbf{0.300}}{0.005}
& \avgain{8}{+7.8\%} \\

&
& P2S
& \pmv{0.839}{0.000}
& \pmv{\textbf{0.793}}{0.004}
& \pmv{0.508}{0.014}
& \pmv{0.269}{0.012}
& \avgain{22}{\textbf{+22.0\%}} \\

\cmidrule(lr){2-8}

&
\multirow[c]{4}{*}{\modellabel{GPT-5.5}}
& Direct
& \pmv{0.839}{0.004}
& \pmv{\underline{0.651}}{0.002}
& \pmv{0.525}{0.000}
& \pmv{\underline{0.266}}{0.009}
& 0.0\% \\

&
& Off-the-shelf
& \pmv{0.855}{0.003}
& \pmv{0.512}{0.004}
& \pmv{\textbf{0.914}}{0.013}
& \pmv{0.264}{0.026}
& \avgain{14}{+13.5\%} \\

&
& LLM-generated
& \pmv{\textbf{0.876}}{0.004}
& \pmv{0.425}{0.016}
& \pmv{\underline{0.883}}{0.008}
& \pmv{\textbf{0.276}}{0.016}
& \avgain{10}{+10.4\%} \\

&
& P2S
& \pmv{\underline{0.868}}{0.016}
& \pmv{\textbf{0.852}}{0.020}
& \pmv{0.861}{0.005}
& \pmv{0.263}{0.009}
& \avgain{24}{\textbf{+24.3\%}} \\

\bottomrule
\end{tabular}

\end{threeparttable}
\vskip -1ex
\end{table*}
\subsection{Main Results}
\label{sec:main_results}
\vspace{-0.75ex}

We report our main results in Table~\ref{tab:main_results}.
\method{} achieves the highest average relative improvement
for every target model, with gains ranging from 22.0\% to
29.1\% for the Qwen and commercial models, and 93.5\% for
Llama-3.2-1B over its three reported benchmarks.
Although \method{} does not achieve the best score in every
setting, its improvements are the most consistent
across the evaluated settings. It also achieves the highest mean score in 10 pairs, including
all five solvers on SQuAD and all four benchmarks for
Qwen3-32B.

Fixed skills can also produce substantial regressions.
For example, the off-the-shelf skill reduces
Llama-3.2-1B's SQuAD score from 0.207 to 0.048, whereas
\method{} raises its SQuAD score to 0.492.
These results illustrate that a skill's usefulness depends
on the target model and task, and
the fixed baselines are constructed or selected without
feedback from target-model executions.

\vspace{-0.75ex}
\paragraph{Notes on Data Overlap and Test Leakage.}
An audit of archived proxy data found no lexical overlap flags
or workbook-hash matches against the target tests;
source details and coverage appear in
Appendix~\ref{app:data_overlap}.

\vspace{-0.75ex}
\subsection{Ablation Studies}
\label{sec:ablation}
\vspace{-0.75ex}

\paragraph{Sensitivity to Initialization.}
Figure~\ref{fig:ablation_seed} compares initializing \method{} with its
default seed skill $s_0$ or an off-the-shelf domain skill on SearchQA and
SQuAD, using Qwen3-8B, Qwen3-32B, and GPT-5.5. Neither initialization
consistently outperforms the other. For example, the off-the-shelf seed
improves SQuAD performance for Qwen3-8B but reduces it for GPT-5.5.
Overall, the two initializations yield broadly comparable performance,
suggesting that \method{} does not depend on a particular seed skill.

\vspace{-0.75ex}
\paragraph{Optional Input--Output Examples.}
We compare zero, one, and four demonstrations with Qwen3-8B on
SearchQA and SQuAD (Figure~\ref{fig:ablation_demonstration}).
Direct prompting uses examples at inference, while \method{} uses
them only during skill construction. Additional examples do not
consistently improve performance: direct prompting on SQuAD improves
from 70.5\% with zero examples to 79.0\% with four, whereas
\method{} remains nearly unchanged on SearchQA between one and
four examples (61.1\% versus 60.9\%). These results support treating
user-provided demonstrations as optional, since \method{} retrieves
task-relevant data to support skill optimization.

\paragraph{Comparing to Skill Optimization}
\begin{wraptable}[7]{r}{0.43\linewidth}
    \centering
    \footnotesize
    \setlength{\tabcolsep}{3pt}
    \caption{Matched-budget results.}
    \vspace{-1.25ex}
    \label{tab:skillopt_comparison}
    \begin{tabular}{@{}lccc@{}}
    \toprule
    Method & \shortstack{SearchQA\\EM}
           & \shortstack{SQuAD\\EM}
           & \shortstack{SQuAD\\F1} \\
    \midrule
    SkillOpt & 45.79          & 81.30          & 90.50 \\
    P2S      & \textbf{48.71} & \textbf{82.50} & \textbf{91.11} \\
    \bottomrule
    \end{tabular}
\end{wraptable}

In order to assess whether the gains extend beyond proxy-data
acquisition, we compare \method{} with SkillOpt~\citep{yang2026skilloptexecutivestrategyselfevolving}
using the same proxy data and initial skill. Both methods use
Qwen3-8B as the solver and GPT-5.5 as the editor, with a cap of
7,400 solver evaluations per task. As shown in
Table~\ref{tab:skillopt_comparison}, \method{} improves SearchQA
EM by 2.93 percentage points and SQuAD EM/F1 by 1.20/0.61 points.
These results support the overall optimization and selection
procedure when the target distribution is unknown, showing the effectiveness of the proposed optimization procedure.

\begin{figure}[t]
    \centering
    \begin{subfigure}[t]{0.48\linewidth}
        \centering
        \includegraphics[width=\linewidth]{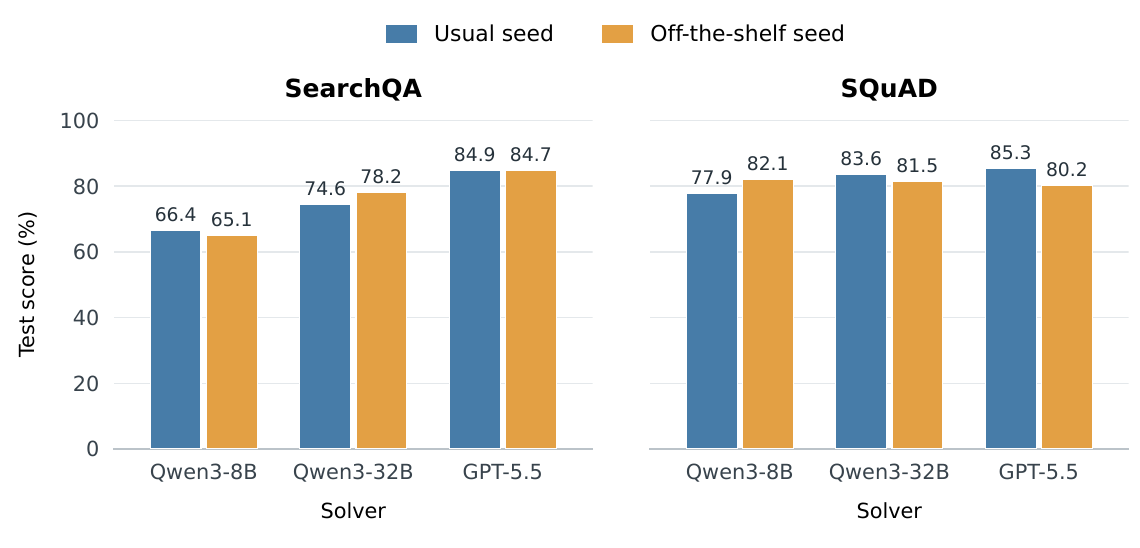}
        \caption{Sensitivity to seed initialization.}
        \label{fig:ablation_seed}
    \end{subfigure}\hfill
    \begin{subfigure}[t]{0.48\linewidth}
        \centering
        \includegraphics[width=\linewidth]{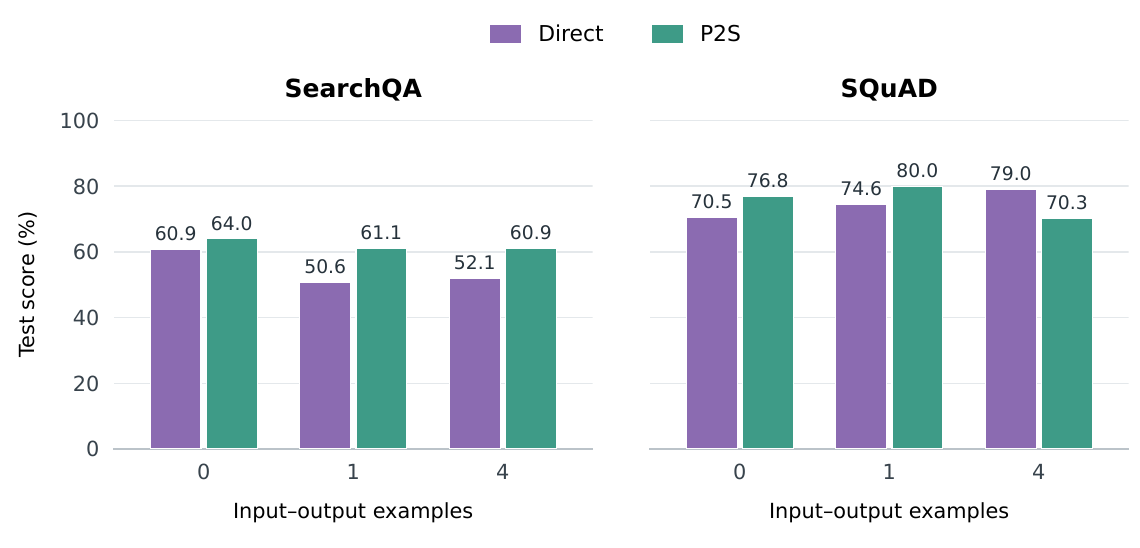}
        \caption{Effect of optional demonstrations.}
        \label{fig:ablation_demonstration}
    \end{subfigure}
    \caption{Ablations on SearchQA and SQuAD. Scores are percentages;
    higher is better. Each ablation condition uses one run. Zero-example
    \method{} scores in (b) are three-run means from the main table;
    the matched zero-example runs failed during data acquisition.}
    \label{fig:combined}
\end{figure}

\vspace{-0.5ex}
\section{Related Work}
\vspace{-0.5ex}
\paragraph{Skill Optimization.}
Automatic prompt optimization improves model behavior by
searching over textual instructions.
APE~\citep{zhou2023largelanguagemodelshumanlevel} generates and selects instruction candidates,
while ProTeGi~\citep{pryzant2023automaticpromptoptimizationgradient} uses natural-language critiques
of prediction errors to guide prompt edits.
GEPA~\citep{agrawal2026gepareflectivepromptevolution} combines reflection on execution feedback
with Pareto-based candidate search.
Recent work extends this approach to reusable agent skills:
Trace2Skill~\citep{ni2026trace2skilldistilltrajectorylocallessons} consolidates lessons from
execution trajectories into transferable skill directories,
while SkillOpt~\citep{yang2026skilloptexecutivestrategyselfevolving} refines skill documents through
bounded edits and validation-gated updates.

\vspace{-0.5ex}
\paragraph{AutoML.}
Recent AutoML systems use language models to automate
machine-learning workflows.
Prompt2Model~\citep{viswanathan2023prompt2modelgeneratingdeployablemodels}
converts task descriptions into deployable models through
dataset and model retrieval, synthetic data generation,
and supervised fine-tuning.
AutoML-Agent~\citep{trirat2025automlagentmultiagentllmframework} coordinates
specialized agents to automate the pipeline from data
retrieval to model deployment, using retrieval-augmented
planning and multi-stage verification.
MLZero~\citep{fang2025mlzeromultiagentendtoendmachine} combines multimodal data
interpretation with memory-guided code generation,
while AIDE~\citep{jiang2025aideaidrivenexplorationspace} formulates machine-learning
engineering as tree search over executable solutions.

\vspace{-0.5ex}
\section{Conclusion}
\vspace{-0.5ex}
In this paper, we introduced \method{}, a framework that converts
natural-language task descriptions into reusable skills for frozen
language models, without requiring a supplied target-task training set.
The Data Agent constructs task-matched proxy data through retrieval,
adaptation, and generation, while the Skill Agent refines skills
through execution feedback, reflective editing, and statistical
screening on fresh batches.
Across four benchmarks and five target models, \method{} achieves
average relative improvements of 36.1\% over direct prompting and
83.0\% over off-the-shelf skills, averaged across the 19 evaluated
model--benchmark pairs, demonstrating that task descriptions alone
can guide the automated construction of effective skills for
frozen language models.

\section{AI Use Disclosure}
In this work, we used generative AI tools for brainstorming research idea, implementing part of the codebase, and formulate theoretical proof. We have not used generative AI tools for autoresearch, propose hypothesis, or provide feedback on methodology, and the rest of the required disclosure tasks are not applicable to this work. Additionally, we used generative AI tools for draft part of the research paper, create or modify the teaser figure and images, and summarize existing landscape in the research area that is within our interests. We have reviewed all AI-assisted work. We checked LLM-generated research ideas for potential plagiarism through a manual literature survey, and audited the claims and code written by LLMs. We take responsibility for the final content of this work, including text, claims or artifacts produced with the aid of generative AI.

\bibliography{iclr2025_conference}

@misc{xu2026agentskillslargelanguage,
      title={Agent Skills for Large Language Models: Architecture, Acquisition, Security, and the Path Forward}, 
      author={Renjun Xu and Yang Yan},
      year={2026},
      eprint={2602.12430},
      archivePrefix={arXiv},
      primaryClass={cs.MA},
      url={https://arxiv.org/abs/2602.12430}, 
}

@misc{yang2026skilloptexecutivestrategyselfevolving,
      title={SkillOpt: Executive Strategy for Self-Evolving Agent Skills}, 
      author={Yifan Yang and Ziyang Gong and Weiquan Huang and Qihao Yang and Ziwei Zhou and Zisu Huang and Yan Li and Xuemei Gao and Qi Dai and Bei Liu and Kai Qiu and Yuqing Yang and Dongdong Chen and Xue Yang and Chong Luo},
      year={2026},
      eprint={2605.23904},
      archivePrefix={arXiv},
      primaryClass={cs.AI},
      url={https://arxiv.org/abs/2605.23904}, 
}

@misc{li2026skillsbenchbenchmarkingagentskills,
      title={SkillsBench: Benchmarking How Well Agent Skills Work Across Diverse Tasks}, 
      author={Xiangyi Li and Yimin Liu and Wenbo Chen and Bingran You and Zonglin Di and Yifeng He and Shenghan Zheng and Kyoung Whan Choe and Jiankai Sun and Shuyi Wang and Chujun Tao and Binxu Li and Xuandong Zhao and Hejia Geng and Xiaojun Wu and Junwei Zhou and Xiaokun Chen and Hanwen Xing and Yubo Li and Qunhong Zeng and Di Wang and Yuanli Wang and Roey Ben Chaim and Penghao Jiang and Haotian Shen and Luyang Kong and Xinyi Liu and Runhui Wang and Xuanqing Liu and Jiachen Li and Xin Lan and Yueqian Lin and Wengao Ye and Junwei He and Songlin Li and Yue Zhang and Yipeng Gao and Yijiang Li and Ze Ma and Liqiang Jing and Tianyu Wang and Kaixin Li and Yiqi Xue and Haoran Lyu and Yizhuo He and Yuchen Tian and Shutong Wu and Bowei Wang and Yixuan Gao and Bo Chen and Litong Liu and Sikai Cheng and Jiajun Bao and Shuaicheng Tong and Shuwen Xu and Terry Yue Zhuo and Tinghan Ye and Qi Qi and Miao Li and Longtai Liao and Zelin Tan and Chang Shi and Xilin Tang and Srinath Tankasala and Boqin Yuan and Yaoyao Qian and Jianhong Tu and Chenguang Wang and Yizhou Sun and Wei Wang and Aaron Taylor and Ziyue Yang and Changkun Guan and Zhikang Dong and Xinyu Zhang and Steven Dillmann and Han-chung Lee and Dawn Song},
      year={2026},
      eprint={2602.12670},
      archivePrefix={arXiv},
      primaryClass={cs.AI},
      url={https://arxiv.org/abs/2602.12670}, 
}

@misc{anthropic2025agentskills,
  author       = {{Anthropic}},
  title        = {Equipping Agents for the Real World with {Agent Skills}},
  year         = {2025},
  month        = oct,
  howpublished = {Anthropic Engineering Blog},
  url          = {https://www.anthropic.com/engineering/equipping-agents-for-the-real-world-with-agent-skills},
  note         = {Accessed 14 August 2026}
}

@misc{ni2026trace2skilldistilltrajectorylocallessons,
      title={Trace2Skill: Distill Trajectory-Local Lessons into Transferable Agent Skills}, 
      author={Jingwei Ni and Yihao Liu and Xinpeng Liu and Yutao Sun and Mengyu Zhou and Pengyu Cheng and Dexin Wang and Erchao Zhao and Xiaoxi Jiang and Guanjun Jiang},
      year={2026},
      eprint={2603.25158},
      archivePrefix={arXiv},
      primaryClass={cs.AI},
      url={https://arxiv.org/abs/2603.25158}, 
}

@misc{dunn2017searchqanewqadataset,
      title={SearchQA: A New Q\&A Dataset Augmented with Context from a Search Engine},
      author={Matthew Dunn and Levent Sagun and Mike Higgins and V. Ugur Guney and Volkan Cirik and Kyunghyun Cho},
      year={2017},
      eprint={1704.05179},
      archivePrefix={arXiv},
      primaryClass={cs.CL},
      url={https://arxiv.org/abs/1704.05179}, 
}

@misc{rajpurkar2016squad100000questionsmachine,
      title={SQuAD: 100,000+ Questions for Machine Comprehension of Text}, 
      author={Pranav Rajpurkar and Jian Zhang and Konstantin Lopyrev and Percy Liang},
      year={2016},
      eprint={1606.05250},
      archivePrefix={arXiv},
      primaryClass={cs.CL},
      url={https://arxiv.org/abs/1606.05250}, 
}

@misc{ma2024spreadsheetbenchchallengingrealworld,
      title={SpreadsheetBench: Towards Challenging Real World Spreadsheet Manipulation}, 
      author={Zeyao Ma and Bohan Zhang and Jing Zhang and Jifan Yu and Xiaokang Zhang and Xiaohan Zhang and Sijia Luo and Xi Wang and Jie Tang},
      year={2024},
      eprint={2406.14991},
      archivePrefix={arXiv},
      primaryClass={cs.CL},
      url={https://arxiv.org/abs/2406.14991}, 
}

@misc{aime25,
      title={American Invitational Mathematics Examination (AIME) 2025}, 
      author={Zhang, Yifan and Math-AI, Team},
      year={2025},
}

@misc{viswanathan2023prompt2modelgeneratingdeployablemodels,
      title={Prompt2Model: Generating Deployable Models from Natural Language Instructions}, 
      author={Vijay Viswanathan and Chenyang Zhao and Amanda Bertsch and Tongshuang Wu and Graham Neubig},
      year={2023},
      eprint={2308.12261},
      archivePrefix={arXiv},
      primaryClass={cs.CL},
      url={https://arxiv.org/abs/2308.12261}, 
}

@misc{zhou2023largelanguagemodelshumanlevel,
      title={Large Language Models Are Human-Level Prompt Engineers}, 
      author={Yongchao Zhou and Andrei Ioan Muresanu and Ziwen Han and Keiran Paster and Silviu Pitis and Harris Chan and Jimmy Ba},
      year={2023},
      eprint={2211.01910},
      archivePrefix={arXiv},
      primaryClass={cs.LG},
      url={https://arxiv.org/abs/2211.01910}, 
}

@misc{pryzant2023automaticpromptoptimizationgradient,
      title={Automatic Prompt Optimization with "Gradient Descent" and Beam Search}, 
      author={Reid Pryzant and Dan Iter and Jerry Li and Yin Tat Lee and Chenguang Zhu and Michael Zeng},
      year={2023},
      eprint={2305.03495},
      archivePrefix={arXiv},
      primaryClass={cs.CL},
      url={https://arxiv.org/abs/2305.03495}, 
}

@misc{fang2025mlzeromultiagentendtoendmachine,
      title={MLZero: A Multi-Agent System for End-to-end Machine Learning Automation}, 
      author={Haoyang Fang and Boran Han and Nick Erickson and Xiyuan Zhang and Su Zhou and Anirudh Dagar and Jiani Zhang and Ali Caner Turkmen and Cuixiong Hu and Huzefa Rangwala and Ying Nian Wu and Bernie Wang and George Karypis},
      year={2025},
      eprint={2505.13941},
      archivePrefix={arXiv},
      primaryClass={cs.MA},
      url={https://arxiv.org/abs/2505.13941}, 
}

@misc{jiang2025aideaidrivenexplorationspace,
      title={AIDE: AI-Driven Exploration in the Space of Code}, 
      author={Zhengyao Jiang and Dominik Schmidt and Dhruv Srikanth and Dixing Xu and Ian Kaplan and Deniss Jacenko and Yuxiang Wu},
      year={2025},
      eprint={2502.13138},
      archivePrefix={arXiv},
      primaryClass={cs.AI},
      url={https://arxiv.org/abs/2502.13138}, 
}

@misc{trirat2025automlagentmultiagentllmframework,
      title={AutoML-Agent: A Multi-Agent LLM Framework for Full-Pipeline AutoML}, 
      author={Patara Trirat and Wonyong Jeong and Sung Ju Hwang},
      year={2025},
      eprint={2410.02958},
      archivePrefix={arXiv},
      primaryClass={cs.LG},
      url={https://arxiv.org/abs/2410.02958}, 
}

@misc{agrawal2026gepareflectivepromptevolution,
      title={GEPA: Reflective Prompt Evolution Can Outperform Reinforcement Learning}, 
      author={Lakshya A Agrawal and Shangyin Tan and Dilara Soylu and Noah Ziems and Rishi Khare and Krista Opsahl-Ong and Arnav Singhvi and Herumb Shandilya and Michael J Ryan and Meng Jiang and Christopher Potts and Koushik Sen and Alexandros G. Dimakis and Ion Stoica and Dan Klein and Matei Zaharia and Omar Khattab},
      year={2026},
      eprint={2507.19457},
      archivePrefix={arXiv},
      primaryClass={cs.CL},
      url={https://arxiv.org/abs/2507.19457}, 
}
\bibliographystyle{iclr2025_conference}

\appendix

\section{Reliable Selection after Adaptive Skill Refinement}
\label{app:selection_guarantee}

Refinement produces skills using execution feedback and comparisons
on different acceptance batches. The final incumbent need not be
the best skill encountered during this process.
This motivates retaining initial and previously accepted skills
and comparing them on a separate selection sample.
The following result bounds the loss from selecting among these
retained skills. Candidate construction may be adaptive, provided
that it remains independent of the selection sample.

\begin{theorem}[Final selection under proxy mismatch]
\label{thm:selection}
Fix the task specification $\tau$, its evaluator
$\widehat{\mu}_{\tau}$, and a proxy distribution
$\widetilde{\mathcal{D}}$ over input--reference pairs.
Let
$\mathcal{S}=(s^{(1)},\ldots,s^{(K)})$
be a nonempty finite collection of eligible skill entries,
constructed independently of the selection sample
$\mathcal{V}\sim\widetilde{\mathcal{D}}^{n}$,
where $n\geq 1$.
If identical skill texts are evaluated multiple times,
each evaluated entry is counted separately in $K$.

Assume that $\mu$ and $\widehat{\mu}_{\tau}$ take values
in $[0,1]$ and that
\begin{equation}
\begin{aligned}
\operatorname{TV}
(\mathcal{D},\widetilde{\mathcal{D}})
&\leq \epsilon_{\mathrm{dist}},\\
\max_{s\in\mathcal{S}}
\mathbb{E}_{(x,y)\sim\widetilde{\mathcal{D}}}
\left[
\left|
\mu(M_s(x),y)
-\widehat{\mu}_{\tau}(M_s(x),y)
\right|
\right]
&\leq \epsilon_{\mathrm{eval}},
\end{aligned}
\label{eq:selection_assumptions}
\end{equation}
where $\operatorname{TV}(P,Q)=\sup_A|P(A)-Q(A)|$.
For stochastic solvers, expectations also include inference
randomness. Selection evaluations use fresh inference randomness,
independent across items and of candidate construction.

Then, for any $\delta\in(0,1)$, the selected skill
$\hat{s}\in\arg\max_{s\in\mathcal{S}}
\widehat{J}_{\mathcal{V}}(s)$ satisfies, with probability
at least $1-\delta$,
\begin{equation}
\begin{aligned}
J(\hat{s})
\geq\;&
\max_{s\in\mathcal{S}}J(s)
-2\bigl(
\epsilon_{\mathrm{dist}}
+\epsilon_{\mathrm{eval}}
\bigr)\\
&-
\sqrt{\frac{2\log(2K/\delta)}{n}}.
\end{aligned}
\label{eq:selection_guarantee}
\end{equation}
\end{theorem}

\paragraph{Proof intuition.}
The argument separates proxy mismatch from selection error.
Distribution and evaluator mismatch bound the difference between
each skill's target value and its population proxy value.
An independent selection sample then controls estimation error
uniformly over the retained candidates.
Together, these bounds relate the empirically selected skill
to the best retained skill on the target distribution.

\begin{proof}
Condition on the realized candidate collection $\mathcal{S}$.
By independence, $\mathcal{V}$ remains an i.i.d.\ sample from
$\widetilde{\mathcal{D}}$.
Define
\[
\widetilde{J}(s)
=
\mathbb{E}_{(x,y)\sim\widetilde{\mathcal{D}}}
\left[
\widehat{\mu}_{\tau}(M_s(x),y)
\right],
\qquad
\epsilon_{\mathrm{proxy}}
:=
\epsilon_{\mathrm{dist}}+\epsilon_{\mathrm{eval}}.
\]
For every $s\in\mathcal{S}$, boundedness of the true metric
and the total-variation assumption give
\begin{equation}
\begin{aligned}
|J(s)-\widetilde{J}(s)|
\leq\;&
\left|
\mathbb{E}_{\mathcal{D}}[\mu(M_s(x),y)]
-
\mathbb{E}_{\widetilde{\mathcal{D}}}[\mu(M_s(x),y)]
\right|\\
&+
\mathbb{E}_{\widetilde{\mathcal{D}}}
\left[
\left|
\mu(M_s(x),y)
-\widehat{\mu}_{\tau}(M_s(x),y)
\right|
\right]\\
\leq\;& \epsilon_{\mathrm{proxy}}.
\end{aligned}
\label{eq:population_proxy_gap}
\end{equation}

For each candidate, its selection score averages $n$ independent
observations in $[0,1]$. Hoeffding's inequality and a union bound
therefore yield
\[
\Pr\left(
\max_{s\in\mathcal{S}}
\left|
\widehat{J}_{\mathcal{V}}(s)-\widetilde{J}(s)
\right|>\alpha
\right)
\leq 2K\exp(-2n\alpha^2).
\]
Taking
$\alpha=\sqrt{\log(2K/\delta)/(2n)}$
makes this probability at most $\delta$.

On the resulting event, let
$s^\star\in\arg\max_{s\in\mathcal{S}}J(s)$.
The empirical optimality of $\hat{s}$ implies
\[
\begin{aligned}
J(\hat{s})
&\geq
\widehat{J}_{\mathcal{V}}(\hat{s})
-\alpha-\epsilon_{\mathrm{proxy}}\\
&\geq
\widehat{J}_{\mathcal{V}}(s^\star)
-\alpha-\epsilon_{\mathrm{proxy}}\\
&\geq
J(s^\star)-2\alpha-2\epsilon_{\mathrm{proxy}}.
\end{aligned}
\]
Substituting $\alpha$ proves the bound.
The argument holds conditional on every candidate collection
satisfying the assumptions, so it also covers adaptive candidate
construction independent of $\mathcal{V}$.
\end{proof}

\paragraph{Implications for the framework.}
Let $\eta_n$ denote the total error allowance in
Eq.~\eqref{eq:selection_guarantee}.
Retaining an initial skill $s_0$ gives
\[
J(\hat{s})\geq J(s_0)-\eta_n.
\]
If refinement discovers a retained skill $s^+$ satisfying
$J(s^+)\geq J(s_0)+\gamma$, then
\[
J(\hat{s})\geq J(s_0)+\gamma-\eta_n.
\]
Thus, terminal selection can preserve an improvement discovered
anywhere along the retained trajectory, up to the stated error
allowance. This motivates keeping the previously accepted skills, rather than exporting only the last
incumbent. It is worth noting that distribution and evaluator mismatch from the data agent remain separate sources of
error.

\section{Application on NLP AutoML}
\label{sec:nlp_automl}
\begin{wrapfigure}{R}{0.5\columnwidth}
    \centering
    \vspace{-6pt}

    \includegraphics[width=\linewidth]{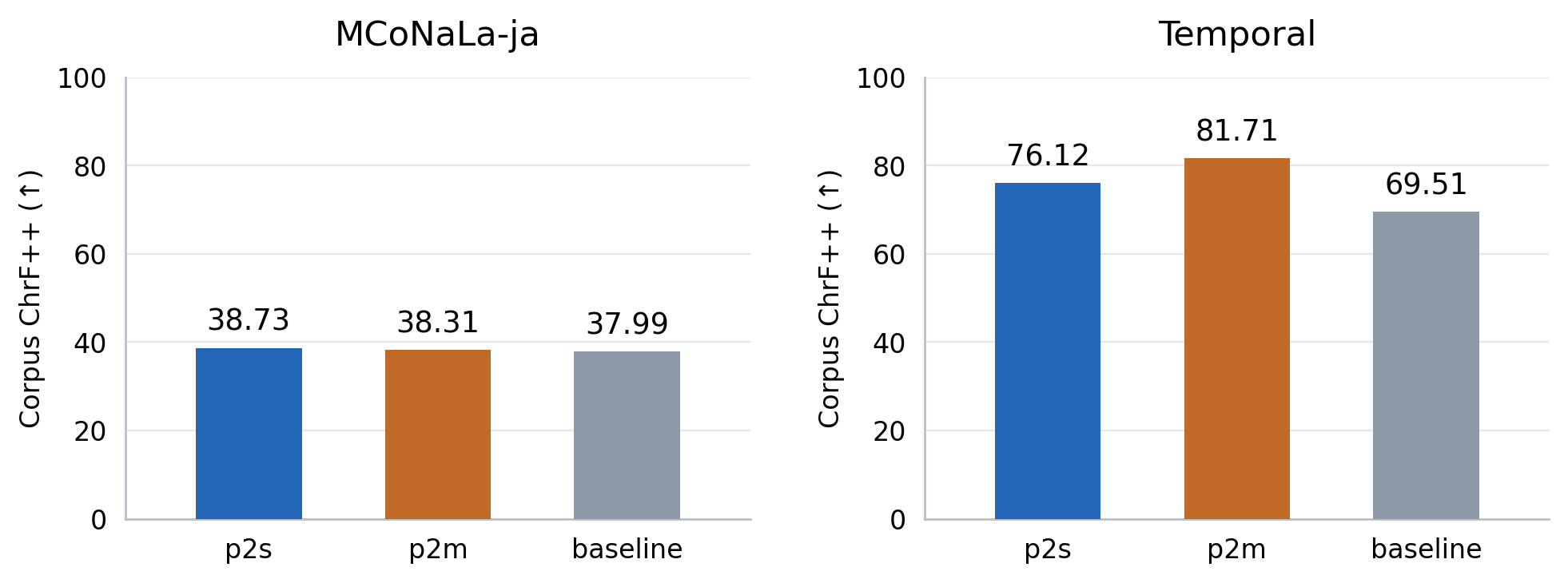}

    \vspace{4pt}
    \caption{AutoML Performance. P2S stands for \method, while P2M stands for Prompt2Model~\citep{viswanathan2023prompt2modelgeneratingdeployablemodels}.}
    \label{fig:automl_nlp}
    \vspace{-6pt}
\end{wrapfigure}
In addition to creating reusable agent skills, Prompt2Skill can be
viewed as an alternative approach to task specialization in NLP
AutoML. Prior work~\citep{viswanathan2023prompt2modelgeneratingdeployablemodels}
introduced Prompt2Model, which converts natural-language task
descriptions into task-specific models through dataset and model
retrieval, synthetic data generation, and supervised fine-tuning.
Prompt2Skill instead produces a reusable skill that guides a frozen
language model, enabling task adaptation without updating its weights.

We explore this application on Japanese-to-Python code generation
(MCoNaLa-ja) and temporal expression normalization (Temporal).
We compare P2S with a modified Prompt2Model baseline that generates
up to 1,000 training examples and fine-tunes a fixed Qwen-32B model. P2S uses Qwen3-32B and
receives a task description with three demonstrations. Both methods are scored on 207 MCoNaLa-ja examples and a 400-example Temporal subset, following ~\cite{viswanathan2023prompt2modelgeneratingdeployablemodels}. Baseline is the plain Qwen-32B model. Figure~\ref{fig:automl_nlp} shows comparable scores for P2S in this
exploratory comparison. These results illustrate the potential of skill-based task
adaptation as a strong alternative to NLP autoML tasks.

\section{Data Sources and Test-Overlap Audit}
\label{app:data_overlap}

\method{} requires no user-supplied training set, but permits automatically
acquired labeled supervision. One concern is that the evaluation might provide an opportunity for hacking the metrics by retrieving the exact data that the framework is evaluated on. We thus audit the retrieved data source across datasets and report the results in Table~\ref{tab:data_overlap_summary}.  

\begin{table}[htbp]
\centering
\small
\setlength{\tabcolsep}{4pt}
\renewcommand{\arraystretch}{1.15}
\caption{Recorded proxy sources and retrospective overlap findings.
Dataset identifiers refer to Hugging Face repositories. Text flags are
exact, containment, or near-text matches under the checks described below.
Workbook matches use two SHA-256 checks; a dash denotes not applicable.
Zero detected matches apply only to the audited artifacts.}
\label{tab:data_overlap_summary}
\begin{tabularx}{\linewidth}{@{}l>{\raggedright\arraybackslash}Xrr@{}}
\toprule
Task & Recorded sources & \shortstack{Text\\flags} & \shortstack{Workbook\\matches} \\
\midrule
SearchQA & \nolinkurl{openaccess-ai-collective/jeopardy} & 0 & --- \\
SQuAD & \nolinkurl{allenai/sciq}; \nolinkurl{hotpotqa/hotpot_qa} & 0 & --- \\
AIME & \nolinkurl{nvidia/OpenMathInstruct-2}; \nolinkurl{SuperSecureHuman/competition_math_hf_dataset} & 0 & --- \\
SpreadsheetBench & Generated workbook tasks & 0 & 0 \\
\bottomrule
\end{tabularx}
\end{table}

\paragraph{Source records.}
Table~\ref{tab:data_overlap_summary} lists configured or discovered source
banks, not verified item-level attribution. For SQuAD, HotpotQA was recorded but had no
bulk-fetch progress in most runs. Source cursors indicate attempted fetching;
they do not identify how many items from each source were retained.

\paragraph{Text checks.}
We normalize text using HTML unescaping, Unicode NFKC normalization,
case folding, and word-token extraction. We check exact questions and
full-question containment, with a minimum of four tokens and 20 characters
for containment. For questions of at least eight tokens, near-text checks
use five-token-shingle Jaccard similarity of at least $0.8$, or a
\texttt{SequenceMatcher} character ratio of at least $0.9$; the latter
also requires a length ratio of at least $0.8$. Near-question comparisons are
restricted to pairs sharing a three-token shingle with the stored input.
For passages of at least 20 tokens, we check full containment or coverage
of at least $80\%$ of the target passage's five-token shingles.
Matching answers alone does not trigger a flag. No audited input was flagged.

\paragraph{Workbook checks.}
The separate spreadsheet audit covers all 945 items in four archived
generation pools, including the initial 40-item draft-construction pool.
It compares 1,884 proxy workbook files with 560 target files from 280 held-out
tasks, using whole-file SHA-256 and a second SHA-256 over sorted workbook
XML and relationship files under \texttt{xl/}. The second check ignores
ZIP metadata and document properties, but not XML serialization differences.
No instruction flags or workbook-hash matches were found. Three manifest
items have no declared workbook files and receive text checks only.

\section{Implementation Details}
\label{app:implementation-details}

\paragraph{Paired acceptance.}
We provide further details on the acceptance rule in
Eq.~\eqref{eq:acceptance_statistic}.
At round $t$, each candidate is compared with the incumbent
on the same fresh batch $\mathcal{B}_t$. This paired evaluation
measures improvements and regressions on identical items,
avoiding differences caused by evaluating skills on separate
batches. The incumbent's outcomes are computed once and
reused across candidate comparisons.

For acceptance item $i$, let
$r_{t,i},r'_{t,j,i}\in\{0,1\}$ denote the correctness of the
incumbent and candidate $j$, respectively.
We count candidate-only successes as wins and incumbent-only
successes as losses:
\begin{equation}
\begin{aligned}
w_{t,j}
&=\sum_{i=1}^{|\mathcal{B}_t|}
r'_{t,j,i}(1-r_{t,i}),\\
\ell_{t,j}
&=\sum_{i=1}^{|\mathcal{B}_t|}
r_{t,i}(1-r'_{t,j,i}).
\end{aligned}
\end{equation}
Items on which both skills succeed or both fail contribute
neither a win nor a loss. Consequently, the empirical accuracy
gain is $(w_{t,j}-\ell_{t,j})/|\mathcal{B}_t|$.

The acceptance statistic scales the win--loss imbalance by
the number of items on which the outcomes differ:
\begin{equation}
z_{t,j}=
\begin{cases}
\dfrac{|w_{t,j}-\ell_{t,j}|-1}
{\sqrt{w_{t,j}+\ell_{t,j}}},
& w_{t,j}+\ell_{t,j}>0,\\[6pt]
0,&\text{otherwise}.
\end{cases}
\label{eq:appendix_acceptance}
\end{equation}
The subtraction of $1$ is a continuity correction that reduces
the statistic for small win--loss differences.
A candidate qualifies only when $w_{t,j}>\ell_{t,j}$ and
$z_{t,j}\geq\kappa$, with $\kappa=1$.
The first condition establishes the direction of improvement;
the second screens the evidence supporting that improvement.
Among qualifying candidates, the agent accepts the one with
the largest net gain $w_{t,j}-\ell_{t,j}$, breaking ties in
candidate-generation order. If none qualifies, the incumbent
is retained.

Fresh batches separate acceptance from the examples used to
propose the current edits. However, $\kappa=1$ is a
per-candidate screening threshold, not a conventional
$5\%$ significance criterion, and the implementation does
not adjust for comparisons across candidates or rounds.
The separate frozen selection set provides a common basis
for choosing among the retained skills after refinement.

\section{Experimental Details}
\label{app:datasets}

\subsection{Benchmark Inputs and Evaluation}
\label{app:benchmark_metrics}

The main experiments construct skills from task descriptions with
$\mathcal{E}=\varnothing$. Benchmark training examples are not supplied
as optimization data. The resulting skills are evaluated on the fixed
target subsets described below. The task descriptions are supplied in Appendix~\ref{app:task_descriptions}. These evaluation subsets are separated
from the retrieved or generated items used for reflection, acceptance, and final selection.

\paragraph{SearchQA~\citep{dunn2017searchqanewqadataset}.}
We use the 1,400-item test partition of the SearchQA split following prior work~\citep{yang2026skilloptexecutivestrategyselfevolving}. The input contains the
question alone; the search snippets provided by the original benchmark
are omitted to match the request for a Jeopardy-question answerer.
We extract the last \texttt{<answer>} block, or use the complete response
when that block is absent. Answers are lowercased, stripped of articles
and punctuation, and whitespace-normalized. 

\paragraph{SQuAD~\citep{rajpurkar2016squad100000questionsmachine}.}
We evaluate 1,000 examples from the SQuAD v1.1 validation
split, obtained from Hugging Face.\footnote{\url{https://huggingface.co/datasets/rajpurkar/squad}}
We shuffle example indices with seed~0 and retain the first
1,000, caching this subset for every method.
Each input contains the full passage followed by its question.
The extracted answer is compared with every accepted reference span.
Our scorer lowercases text, removes articles, replaces
non-alphanumeric characters with spaces, and collapses whitespace;
an item passes if the resulting prediction equals at least one
normalized reference.
We report this exact-match score. Token F1 is also recorded
but is not the main-table metric.

\paragraph{AIME~\citep{aime25}.}
We use 90 problems from the AI-MO AIME validation
collection\footnote{\url{https://huggingface.co/datasets/AI-MO/aimo-validation-aime}}
and 30 from Math-AI's AIME 2025
collection.\footnote{\url{https://huggingface.co/datasets/math-ai/aime25}}
All 120 problems are reserved for target evaluation,
regardless of their original partition labels.
The input is the problem statement. We extract the last
\texttt{<answer>} block, falling back to the last
\texttt{\textbackslash boxed\{\}} expression or the final
response line.
The scorer removes whitespace, currency delimiters, commas,
and leading zeros before comparing the prediction with the
reference. It also accepts numerically equivalent
representations within $10^{-6}$.
We report the fraction of correctly answered problems.

\paragraph{SpreadsheetBench~\citep{ma2024spreadsheetbenchchallengingrealworld}.}
We obtain the verified 400-task SpreadsheetBench collection
from the Trace2Skill release~\citep{ni2026trace2skilldistilltrajectorylocallessons}.\footnote{\url{https://github.com/Qwen-Applications/Trace2Skill}}
We follow the partition used by SkillOpt~\citep{yang2026skilloptexecutivestrategyselfevolving},
which contains 80 training, 40 validation, and 280 test tasks,
and use the held-out test partition for target evaluation.

\subsection{Models, Inference, and Refinement Budgets}
\label{app:implementation}

\paragraph{Model roles.}
The target solvers are Qwen3-8B, Qwen3-32B,
Llama-3.2-1B-Instruct, Claude Haiku~4.5, and GPT-5.5.
The target solver executes candidates and supplies the outcomes used
for acceptance and selection. GPT-5.5 supplies reflective feedback in
the reported optimization runs. The local inference configurations use a 32,768-token context and disable Qwen's thinking mode. Commercial endpoints are identified by
the model names used in the requests.

\paragraph{Decoding and tools.}
For the three text benchmarks, the skill body is placed in the system
message and the task input in the user message. YAML packaging
metadata is removed. Test-time output limits are 300 tokens for
SearchQA, 600 for SQuAD, and 7,000 for AIME. Spreadsheet methods share the same tools, preloaded-skill interface,
and 40-turn interaction budget. 

\paragraph{Proxy batch schedule.}
The standard text loops request 400 initial proxy items, shuffle them
with seed~0, reserve 150 for terminal selection, and use the remaining
250 for reflection. Each subsequent round requests 150 additional
reflection items. Each candidate is compared with the incumbent on
the same newly acquired acceptance batch of 120 items.
These are requested sizes: source exhaustion and filtering can reduce
the realized batch size. The spreadsheet loop instead requests 22
initial generated tasks, reserves 12 for selection, and requests eight
new reflection tasks and ten acceptance tasks per round.
Generation sessions may yield more tasks than requested, and the
implementation retains that full yield. 

\subsection{Baseline Construction}
\label{app:baselines}

\paragraph{Direct.}
The text baseline supplies a minimal task instruction and the required
answer-tag interface. SearchQA and SQuAD use
``Answer the question. Reply with only the answer, inside
\texttt{<answer></answer>} tags.''
The reported AIME Direct condition uses
``Give the final answer to the problem inside
\texttt{<answer></answer>} tags.''
For spreadsheets, the additional skill body is empty and the common
execution harness remains in place. Thus, Direct retains the
application's input, output, and tool instructions.

\paragraph{Off-the-shelf.}
We select externally published skills for domain relevance and
harness compatibility before evaluating them.
SearchQA and SQuAD share the \texttt{general-reasoning} skill,
and AIME uses \texttt{math-logic-reasoning}, both obtained
from the same public skill repository.\footnote{Revision
\texttt{c7fc50a45bd3}:
\url{https://github.com/ahoynodnarb/reasoning-based-skills/tree/c7fc50a45bd3b8368ecd650eabb58ab37440e6de}}
SpreadsheetBench uses the publicly released
\texttt{xlsx-manipulation} skill.\footnote{Revision
\texttt{9c4c7d5cd281}:
\url{https://github.com/claude-office-skills/skills/tree/9c4c7d5cd2813a8936bf2c9fdb174ea883b85a11}}
These fixed skills are shared across target models.
Referenced Markdown guidance is included in full because the
text solver cannot open additional files, and the source
instruction bodies are preserved.
The complete evaluated skill texts are reproduced in
Appendix~\ref{app:prompts}. The math skill includes general references to AIME/AMC,
which are retained. The source skills' development-data
provenance and human authorship have not been independently
established.

\paragraph{LLM-generated.}
Qwen3-32B authors one skill per domain from a frozen domain brief,
without demonstrations, benchmark names, execution feedback,
off-the-shelf skill text, or iterative selection.
The three domains are factual question answering and reading
comprehension, mathematical reasoning, and spreadsheet manipulation.
Each authoring request asks for one self-contained \texttt{SKILL.md}
of approximately 600--1,000 words and uses temperature~0, seed~0,
and a 6,000-token output limit. The same generated artifacts are
reused across target models. Their authoring prompts and the framework's task prompts are provided
in Appendix~\ref{app:prompts}.

\subsection{Aggregation of Results}
\label{app:aggregation}

Let $p_{m,b}^{A}$ denote the reported mean score for method $A$,
target model $m$, and benchmark $b$. For model $m$, the main table
reports the arithmetic mean of relative changes over its available
benchmarks $\mathcal{B}_m$:
\begin{equation}
 \operatorname{Avg}\Delta_m(A)
 = \frac{100}{|\mathcal{B}_m|}
   \sum_{b\in\mathcal{B}_m}
   \frac{p_{m,b}^{A}-p_{m,b}^{\mathrm{Direct}}}
        {p_{m,b}^{\mathrm{Direct}}}.
 \label{eq:appendix_model_relative_gain}
\end{equation}
There are four benchmarks for each model except Llama-3.2-1B,
whose average covers the three reported text benchmarks.
For the aggregate improvement over comparator $A$ in the conclusion,
we give equal weight to the 19 available model--benchmark pairs
$\mathcal{C}$:
\begin{equation}
 \operatorname{Gain}(A)
 = \frac{100}{|\mathcal{C}|}
   \sum_{(m,b)\in\mathcal{C}}
   \frac{p_{m,b}^{\mathrm{P2S}}-p_{m,b}^{A}}
        {p_{m,b}^{A}}.
 \label{eq:appendix_overall_relative_gain}
\end{equation}
Using the displayed table means gives 36.1\% over Direct and
83.0\% over Off-the-shelf. These are averages of relative changes,
not percentage-point accuracy gains or relative changes in a pooled
accuracy. Equal weighting also means that a pair with a small
baseline score can contribute a large relative change.

\section{Task Descriptions and Prompt Templates}
\label{app:prompts}
This section records the task descriptions and prompt
text in the accompanying implementation. In the listings, \texttt{<<...>>} denotes content inserted at runtime,
not text sent literally to the model. For example,
\texttt{<<skill>>} is the current skill and
\texttt{<<score:.0\%>>} is its score formatted as a whole-number percentage. The notation in the method section groups the following prompts:
\begin{center}
\small
\begin{tabular}{@{}p{0.22\linewidth}p{0.73\linewidth}@{}}
\textbf{Prompt symbol} & \textbf{Implementation templates} \\
$\theta_{\mathrm{setup}}$ & Automatic text task setup; spreadsheet specification and consistency check. \\
$\theta_{\mathrm{retrieve}}$ & Dataset task-category selection and hypothetical dataset description; the tool-using agent also receives the retrieval instructions in its data-agent prompt. \\
$\theta_{\mathrm{adapt}}$ & Source relevance and column selection, followed by the code-based row conversion. \\
$\theta_{\mathrm{validate}}$ & Converted-item compatibility report; structural checks and rollout-based screens are separate code operations. \\
$\theta_{\mathrm{gen}}$ & Text-item generation, or spreadsheet data-agent instructions plus the generation session request. \\
$\theta_{\mathrm{edit}}$ & Reflective editing for the corresponding pipeline, implemented directly or through diagnosis and target-model rewriting. \\
$\theta_{\mathrm{diagnose}}$, $\theta_{\mathrm{rephrase}}$ & Diagnosis and target-model rewriting in the generic text or SearchQA pipeline. \\
\end{tabular}
\end{center}

\subsection{Task Descriptions}
\label{app:task_descriptions}
The main experiments supply the following descriptions without
input--output demonstrations. The optional-example study appends examples
during task setup, as specified below.

\begin{PTwoSPrompt}{SearchQA: task description}
\begin{Verbatim}[breaklines,breakanywhere,fontsize=\scriptsize]
I want a model for answering Jeopardy! questions.
\end{Verbatim}
\end{PTwoSPrompt}

\begin{PTwoSPrompt}{SQuAD: task description}
\begin{Verbatim}[breaklines,breakanywhere,fontsize=\scriptsize]
I want a model that reads a passage and answers questions about it.
\end{Verbatim}
\end{PTwoSPrompt}

\begin{PTwoSPrompt}{AIME: task description}
\begin{Verbatim}[breaklines,breakanywhere,fontsize=\scriptsize]
I want a model that answers math olympiad questions.
\end{Verbatim}
\end{PTwoSPrompt}

\begin{PTwoSPrompt}{SpreadsheetBench: task description}
\begin{Verbatim}[breaklines,breakanywhere,fontsize=\scriptsize]
Develop a model that performs edits on Excel spreadsheets according to natural-language instructions. The input is an Excel file plus a user request (like a forum post) describing the desired change — filtering rows, matching data across sheets, computing aggregates, sorting, conditional edits, reformatting; the model edits the workbook and saves the result, which is checked against the correctly edited file.
\end{Verbatim}
\end{PTwoSPrompt}

\subsection{Automatic Text Pipeline}
\label{app:text_prompts}
\paragraph{Task setup.}
The setup request jointly produces the task specification and seed skill.
The selected matching mode is implemented by code rather than an LLM judge.

\begin{PTwoSPrompt}{Task setup}
\begin{Verbatim}[breaklines,breakanywhere,fontsize=\scriptsize]
A user wants a model built. Their request:
"<<prompt>>"

Derive the task setup. Reply with JSON only:
{"task_description": "<one paragraph: what the model receives as input and
must produce as output>",
 "input_desc": "<one sentence: what one input item looks like>",
 "match_mode": "span_exact" | "numeric" | "containment",
 "request_language": "<ISO 639-1 code of the language the request implies
for the data — the language the request is written in unless it names
another>",
 "expected_answer_style": "short-text" | "number" | "option-letter" |
"long-text" | "code" | "label",
 "task_type": "factual-recall" | "reading-comprehension" |
"math-problem-solving" | "coding" | "translation" | "classification" |
"summarization" | "other",
 "seed_skill": "<a minimal 1-2 sentence system prompt for the task, ending
by instructing: put the final answer inside <answer></answer> tags>"}

match_mode guide: numeric = answers are ONLY ever numbers — pure math
problem-solving tasks (numeric mode also grants the long reasoning budget
such tasks need). If answers MIX numbers with words or text spans (e.g.
passage QA where some questions need counting or dates), use span_exact —
its normalization matches numbers too, and the task_type is then
reading-comprehension, not math-problem-solving. span_exact = short text
spans matched exactly (after case/article/punctuation normalization);
containment = free-form short answers judged by whether the reference
appears in the reply.
\end{Verbatim}
\end{PTwoSPrompt}

\paragraph{Retrieval and adaptation.}
The automatic pipeline requests Hugging Face task tags, generates a
hypothetical dataset card with search phrases, and also searches using the
first 80 characters of the original task description. These strings are
passed to retrieval tools; search itself is an external operation.
For each candidate source, the column-selection prompt examines retrieved
sample rows. Its selected columns are checked against the actual schema,
and accepted rows are converted by the deterministic row mapper.

\begin{PTwoSPrompt}{Dataset task-category selection}
\begin{Verbatim}[breaklines,breakanywhere,fontsize=\scriptsize]
Task: <<task>>

Which HuggingFace dataset task-category tags fit this task? Choose 1-2 from exactly this list: question-answering, text-classification, summarization, translation, text-generation, table-question-answering, visual-question-answering, document-question-answering, token-classification, sentence-similarity, multiple-choice, text2text-generation. Reply JSON only: {"tags": ["..."]}
\end{Verbatim}
\end{PTwoSPrompt}

\begin{PTwoSPrompt}{Hypothetical dataset description and search phrases}
\begin{Verbatim}[breaklines,breakanywhere,fontsize=\scriptsize]
Write a short hypothetical HuggingFace dataset card (title + 3-sentence description + typical column names) for a dataset that would be PERFECT for this task:
<<task_description>>
Then on the final line write SEARCHES: followed by 4 comma-separated short search phrases for finding such datasets.
\end{Verbatim}
\end{PTwoSPrompt}

\begin{PTwoSPrompt}{Source relevance and column selection}
\begin{Verbatim}[breaklines,breakanywhere,fontsize=\scriptsize]
Analyze whether this dataset fits the task, and if so which columns are the INPUT and which single column is the OUTPUT. It is completely acceptable — and important — to answer that the dataset is irrelevant if no columns fit the task.

TASK: <<task_description>>

DATASET <<dataset_id>> sample rows:
<<peek[:3000]>>

Reply with JSON only: {"verdict": "relevant"|"irrelevant"|"ambiguous", "input_columns": [..], "output_column": "..", "reason": ".."}
\end{Verbatim}
\end{PTwoSPrompt}

\paragraph{Task compatibility.}
The next prompt reports observed properties of converted samples. Code
compares these properties with the previously established specification.
The displayed sample contains up to three inputs, each truncated to 500
characters, and the first reference answer truncated to 100 characters.
A separate answerability screen executes the seed-equipped target model
on a small source sample; it does not use a second LLM judging prompt.

\begin{PTwoSPrompt}{Converted-item compatibility report}
\begin{Verbatim}[breaklines,breakanywhere,fontsize=\scriptsize]
Candidate training items from dataset <<did>>:
<<shown>>

Report facts about these items. Reply with JSON only:
{"item_language": "<ISO 639-1 code of the language the candidate items are written in>",
 "item_answer_style": "<what these ITEMS' expected answers are — one of: short-text, number, option-letter, long-text, code, label. option-letter means the answer is a choice key like A/B/C, even if answer options are listed in the input>",
 "item_task_type": "<what skill these items exercise — one of: factual-recall, reading-comprehension, math-problem-solving, coding, translation, classification, summarization, other>",
 "input_complete": <true if each item's INPUT contains every component this task input requires: <<self.state['spec']['input_desc']!r>> — false if any required component (e.g. a passage, a document, an image reference) is absent from the items>,
 "reason": "..."}
\end{Verbatim}
\end{PTwoSPrompt}

\paragraph{Generation.}
When synthesis is invoked, the prompt includes up to six existing items
as format references, together with the latest diagnosis when available.
The input of each displayed example is truncated to 800 characters.
The generation prompt does not contain the seed skill. After generation,
seed and direct-prompt rollouts are used to screen candidate items.

\begin{PTwoSPrompt}{Text-item generation}
\begin{Verbatim}[breaklines,breakanywhere,fontsize=\scriptsize]
You create NEW practice items for training a model on this task.
Task: <<prompt>>
<<task>>

Real examples of this task's items (format reference):
<<examples>>

The model in training shows these weaknesses (analyst diagnosis):
<<diagnosis>>

Write <<n>> NEW items of the same kind and format that specifically exercise
these weaknesses and are moderately harder than the examples. Rules:
- Each item must be fully self-contained and objectively answerable from
  its own content. Verify each answer yourself before emitting it; only
  emit items whose answer you are certain of.
- Do not copy the examples; vary topics and structure.
- Keep the input format identical to the examples.
Reply with JSON only:
[{"input": "<item text>", "answers": ["<answer>"]}, ...]
\end{Verbatim}
\end{PTwoSPrompt}

If no examples are available, the example block is replaced by the
following literal text.

\begin{PTwoSPrompt}{Generation: empty-example fallback}
\begin{Verbatim}[breaklines,breakanywhere,fontsize=\scriptsize]
(no examples available - invent representative items that faithfully match the task description, varied in topic, phrasing, and difficulty)
\end{Verbatim}
\end{PTwoSPrompt}

When no diagnosis is available, its placeholder is replaced with
\texttt{(no diagnosis yet - target generally challenging variations of the examples)}.
The direct screening instruction is \texttt{Give the final answer inside
<answer></answer> tags.}

\paragraph{Reflective editing.}
The automatic text loop supports direct editing and diagnosis followed by
rewriting. The reported SQuAD and AIME runs use direct editing.
The failure block contains at most eight failures; each shows a 200-character
input excerpt, its first reference answer, and a 60-character excerpt of the
model output. The score in these prompts is measured on the reflection set,
which is distinct from the terminal selection set, despite the historical
wording ``held-out'' in the templates.

\begin{PTwoSPrompt}{Automatic text loop: direct reflective editing}
\begin{Verbatim}[breaklines,breakanywhere,fontsize=\scriptsize]
You are refining the SYSTEM PROMPT (skill) of a model.
Task: <<prompt>>
<<input_desc>>

CURRENT SKILL:
---
<<skill>>
---
Score of the current skill on held-out real examples: <<score:.0%>>.

Failures of the CURRENT skill (input excerpt, correct answer, model's wrong
answer):
<<failures>>

Write an IMPROVED complete skill. STRICT RULES for what a skill may contain:
- Allowed: PROCESS instructions — how to read and decompose the input, how
  to derive candidate answers, how to verify them against every constraint
  in the input, how to choose the expected answer GRANULARITY and exact
  wording/format, when to re-check or restart, how to format output.
  Reasoning before answering is allowed if the final answer is inside
  <answer></answer> tags.
- FORBIDDEN: any rule naming specific topics, entities, domains, or
  categories; any rule of the form "when the input is about X, do Y"; rules
  must plausibly help THOUSANDS of unseen inputs, not fix the examples
  above.
Keep it under 40 lines. Output ONLY the new skill text.
\end{Verbatim}
\end{PTwoSPrompt}

\subsection{LLM-Generated Baseline}
\label{app:baseline_authoring_prompts}
The LLM-generated baseline is authored once per domain from the following
system instruction and domain brief. The same factual-QA skill is used for
SearchQA and SQuAD. These authoring requests contain no proxy rollouts or
target-model feedback. The generated skills are held fixed across target
solvers; the authoring model is Qwen3-32B.

\begin{PTwoSPrompt}{Baseline authoring: system message}
\begin{Verbatim}[breaklines,breakanywhere,fontsize=\scriptsize]
Write a reusable domain skill for a language-model assistant from the domain brief alone.
Return only a complete SKILL.md file with YAML frontmatter containing name and description, followed by practical instructions.
The skill must be self-contained in this one file. Do not require other skills, reference files, browsing, external services, or tools beyond those specified in the brief.
Give broadly useful domain procedures, checks, and pitfalls. Do not mention datasets, benchmarks, evaluation metrics, model identities, or particular test conventions. Do not include example questions, worked examples, or example answers. Do not prescribe an answer schema; the calling application supplies its output-format requirements.
Aim for 600 to 1000 words. Produce one finished skill; do not describe alternatives or a revision process.
\end{Verbatim}
\end{PTwoSPrompt}

\begin{PTwoSPrompt}{Baseline authoring: factual QA and reading comprehension}
\begin{Verbatim}[breaklines,breakanywhere,fontsize=\scriptsize]
Domain: factual question answering and reading comprehension.
The assistant answers questions across general knowledge domains. A question may stand alone or be accompanied by a passage. It should use the supplied passage when present and its existing knowledge when no passage is supplied. The assistant has no browsing, retrieval, code execution, or file-reading tools. Write a reusable skill for accurate, precise answers under these conditions.
\end{Verbatim}
\end{PTwoSPrompt}

\begin{PTwoSPrompt}{Baseline authoring: mathematics}
\begin{Verbatim}[breaklines,breakanywhere,fontsize=\scriptsize]
Domain: mathematical problem solving.
The assistant solves mathematical questions involving arithmetic, algebra, geometry, number theory, combinatorics, and probability, including difficult competition-style problems. It has no browsing, calculator, code execution, or file-reading tools. Write a reusable skill for choosing strategies, carrying out reasoning and calculations, handling constraints and cases, and verifying conclusions.
\end{Verbatim}
\end{PTwoSPrompt}

\begin{PTwoSPrompt}{Baseline authoring: spreadsheets}
\begin{Verbatim}[breaklines,breakanywhere,fontsize=\scriptsize]
Domain: spreadsheet creation, inspection, and editing.
The assistant works on local Excel workbooks and tabular files in response to user instructions. It can use a shell and Python with openpyxl and pandas. Write a reusable skill for inspecting workbook structure, reading and transforming data, working with formulas and formatting, preserving unrelated workbook contents, saving the requested output, and verifying the result. It has no browsing or external spreadsheet service.
\end{Verbatim}
\end{PTwoSPrompt}

For both off-the-shelf and LLM-generated text skills, the following
application-level instruction is prepended and appended to the skill body.
The spreadsheet baselines use the common spreadsheet execution harness
without this text-answer adapter.

\begin{PTwoSPrompt}{Shared output-format adapter for fixed text skills}
\begin{Verbatim}[breaklines,breakanywhere,fontsize=\scriptsize]
Application output format: put the final answer only inside <answer></answer> tags at the end of the response. Keep any explanation outside those tags. Follow this output format if a skill below describes a different presentation format. Only the capabilities supplied by the application are available.
\end{Verbatim}
\end{PTwoSPrompt}

\subsection{Off-the-Shelf Skills}
\label{app:off_the_shelf_skills}
\begingroup
\setlength{\emergencystretch}{3em}
We reproduce the complete instruction bodies used by the off-the-shelf
baselines. SearchQA and SQuAD share \texttt{general-reasoning}, while
AIME uses \texttt{math-logic-reasoning}; both are from the
\texttt{reasoning-based-skills} repository at revision
\texttt{c7fc50a45bd3}.\footnote{\url{https://github.com/ahoynodnarb/reasoning-based-skills/tree/c7fc50a45bd3b8368ecd650eabb58ab37440e6de}}
SpreadsheetBench uses \texttt{xlsx-manipulation} from the
Claude Office Skills repository at revision
\texttt{9c4c7d5cd281}.\footnote{\url{https://github.com/claude-office-skills/skills/tree/9c4c7d5cd2813a8936bf2c9fdb174ea883b85a11}}
Each baseline skill is shared across target models. The listings include
bundled Markdown references and the answer-format adapters used in evaluation.
YAML packaging metadata is omitted, matching the skill loaders; line endings
and terminal blank lines are normalized for display. The instruction bodies
are otherwise preserved.
\par
\endgroup

\begin{PTwoSSkill}{PTwoSSkillBlue}{General reasoning: SearchQA and SQuAD}
\begin{Verbatim}[breaklines,breakanywhere,fontsize=\scriptsize]
Application output format: put the final answer only inside <answer></answer> tags at the end of the response. Keep any explanation outside those tags. Follow this output format if a skill below describes a different presentation format. Only the capabilities supplied by the application are available.



# General Reasoning Performance Skill

This skill helps the agent produce its highest-quality answers on reasoning-intensive tasks. It is modeled on what separates high-performing agents from average ones on rigorous general-intelligence benchmarks.

## Core Philosophy

These benchmarks don't test memorization — they test **reliable reasoning under uncertainty**. The key failure modes are:
- Confidently selecting a plausible-sounding wrong answer
- Skipping careful analysis because the question *looks* easy
- Anchoring on the first interpretation rather than the correct one
- Collapsing uncertainty prematurely

The antidote is a consistent, disciplined process applied to every non-trivial question.

---

## The Reasoning Protocol

### Step 1: Fully Parse the Question Before Answering

Before generating any answer content, make sure you've understood:
- **What is literally being asked?** (Not what you expect to be asked.)
- **What domain and subfield is this in?**
- **What level of precision is required?** (Ballpark vs. exact vs. formal proof)
- **Are there any tricks, negations, or double-negatives in the question?** ("Which of the following is NOT..." is different from "Which of the following is...")
- **What would a wrong answer look like?** Anticipate the distractors.

> **Red flag**: If you start generating an answer within the first second of reading a hard question, you haven't done this step.

### Step 2: Activate the Right Knowledge Frame

For each question, mentally identify:
- The primary concept being tested
- The standard result, definition, or framework relevant to it
- Any common misconceptions or edge cases in this area
- Whether this requires recall, derivation, or judgment

Don't conflate adjacent concepts. E.g., "entropy" in thermodynamics vs. information theory are related but distinct.

### Step 3: Reason Before Concluding

For hard questions, produce explicit intermediate reasoning. Do not jump to an answer:

- **For factual questions**: Recall what you know, note your confidence level, consider whether you might be confusing this with something adjacent.
- **For multi-step problems**: Write out the chain of reasoning. Don't compress steps. Each step should follow from the last.
- **For multiple choice**: Analyze each option independently before comparing. Don't pick the "best-sounding" option — eliminate options based on what's *wrong* about them.
- **For ambiguous questions**: Identify the most reasonable interpretation AND note if there's a less obvious interpretation that could change the answer.

### Step 4: Stress-Test Your Answer

Before committing, briefly check:
- Does this answer make sense dimensionally / logically / empirically?
- Is there a known result, theorem, or fact that would confirm or contradict it?
- Have I seen a question that looks like this but has a counterintuitive answer?
- If there are multiple choice options, does your answer actually match one of them exactly? (Don't approximate.)

### Step 5: Calibrate Your Confidence

State your confidence when it matters. If you're uncertain:
- Say so clearly.
- Name what would change your answer.
- Offer the top 2 candidates if genuinely unsure.

Do NOT fake certainty. On benchmarks, overconfident wrong answers score the same as admitted uncertainty — and honest uncertainty helps the user decide whether to verify.

---

## Domain-Specific Tactics

Read `/references/domain-tactics.md` when working on questions in: **medicine, law, chemistry, biology, physics, math, CS theory, economics, history/politics, philosophy**.

Each domain has specific failure modes and best practices that differ from generic reasoning.

---

## Multiple Choice Questions

Multiple choice deserves special treatment because distractors are engineered to exploit reasoning shortcuts.

**The two-pass method:**

**Pass 1 — Elimination**: Go through each option and ask "Is this definitely wrong, and why?" Mark options you can eliminate with confidence.

**Pass 2 — Selection**: Among remaining options, identify which is most precisely correct — not just most plausible.

**Common traps:**
- Options that are *true statements* but don't answer the question asked
- Options with one word that makes them false (read carefully)
- "All of the above" / "None of the above" — require you to have verified all others
- Magnitude traps (e.g., "increases" vs. "increases significantly")
- Temporal/causal confusion ("A causes B" vs. "A is correlated with B")

---

## Handling Uncertainty and Knowledge Gaps

When you don't know something with confidence:

1. **Reason from first principles** — what can be derived from what you do know?
2. **Use analogical reasoning carefully** — similar domains can be informative but also misleading
3. **Bound the answer** — even if you don't know the exact answer, you may be able to rule out extreme options
4. **Be honest** — "I'm not certain, but my best reasoning leads to X because Y" is more valuable than false confidence

Never hallucinate citations, specific statistics, or precise technical facts to fill gaps. Say "I don't have reliable recall of the exact figure, but..."

---

## Meta-Cognitive Checks

These are quick sanity checks to run on your own reasoning process:

| Check | What to ask |
|-------|-------------|
| **Anchoring** | Am I committed to my first instinct without re-examining? |
| **Framing** | Did I take the question at face value? Are there other readings? |
| **Completeness** | Did I actually address all parts of the question? |
| **Precision** | Is my answer specific enough, or vague in ways that matter? |
| **Distractor check** | For MCQ: did I pick the answer that *sounds* best or the one that *is* best? |

---

## Output Format Guidelines

**For open-ended reasoning questions:**
- Lead with your reasoning chain, then state the conclusion
- Don't bury the answer in the middle of dense prose
- If multiple steps, number them

**For multiple choice:**
- State which option you select and why
- Briefly explain why the strongest distractor is wrong (shows you checked)

**For scientific/technical questions:**
- Use standard notation and terminology for the domain
- Show units, constraints, and assumptions where relevant
- If deriving, show the key derivation steps — not just the result

**For questions with definitive correct answers:**
- Be direct. Don't hedge when you're confident.
- Don't pad with caveats that add no information

---

## Quick Reference Checklist

Before submitting any answer to a hard question, confirm:
- [ ] I fully parsed the question (including any negations/qualifications)
- [ ] I identified the relevant domain and conceptual framework
- [ ] I reasoned step-by-step rather than pattern-matching to an answer
- [ ] I stress-tested my answer against known results or logical consistency
- [ ] My confidence level is stated and calibrated
- [ ] My answer format matches what was asked

---

*See `/references/domain-tactics.md` for domain-specific guidance.*




# Bundled reference: references/domain-tactics.md

# Domain-Specific Reasoning Tactics

This file contains targeted guidance for domains where benchmark questions frequently appear. Each section covers: the core framework to activate, common traps, and what distinguishes a correct answer from a plausible-but-wrong one.

---

## Medicine & Biology

**Core framework**: Mechanistic biological thinking — trace the causal chain from stimulus to response, or pathology to symptom.

**Key tactics:**
- For clinical questions: Use the standard diagnostic framework (chief complaint → pathophysiology → presentation → treatment). Don't skip steps.
- Distinguish *mechanism* from *association*. A drug that "lowers blood pressure" may do so via multiple pathways — the question may be asking specifically which one.
- For pharmacology: Know MOA (mechanism of action), not just drug class. "Beta blocker" is not enough if the question is about selectivity or downstream effects.
- Genetics questions: Distinguish penetrance, expressivity, epistasis. Autosomal dominant ≠ 100% expression.

**Common traps:**
- Confusing sensitivity with specificity (and PPV/NPV with those)
- Mixing up first-line vs. gold standard vs. most common treatment
- Forgetting exceptions to classic presentations (e.g., atypical MI in women/diabetics)
- Relative risk vs. absolute risk vs. odds ratio — these are not interchangeable

---

## Chemistry

**Core framework**: Formal chemical reasoning — electron flow, thermodynamic vs. kinetic control, and conservation laws.

**Key tactics:**
- For organic mechanisms: Draw the electron flow, don't guess product from pattern memory alone
- Thermodynamics: ΔG = ΔH - TΔS. At high T, entropy dominates; at low T, enthalpy dominates
- Distinguish equilibrium (ΔG) from rate (activation energy, Ea). A favorable ΔG doesn't mean a fast reaction
- Electrochemistry: Know the sign conventions (reduction potential, cell notation direction)
- Spectroscopy: IR → functional groups; NMR → connectivity and H count; MS → molecular weight and fragmentation

**Common traps:**
- Le Chatelier's principle: Adding inert gas at constant volume doesn't shift equilibrium; at constant pressure it does (via partial pressure changes)
- Solubility rules applied incorrectly under non-standard conditions
- Acid/base in non-aqueous solvents behaves differently
- "More stable" ≠ "more reactive" (thermodynamic vs. kinetic product)

---

## Physics

**Core framework**: Identify the relevant physical law, apply it carefully with sign conventions and constraints.

**Key tactics:**
- Dimensional analysis first — if the units don't work out, your approach is wrong
- Draw a free body diagram or energy diagram before writing equations
- For E&M: Gauss's law, Faraday's law, and Ampere's law require careful choice of Gaussian surface / Amperian loop
- Quantum: Distinguish expectation values from eigenvalues. "The measured value" is an eigenvalue; "the average value" is an expectation value
- Relativity: Events that are simultaneous in one frame are not in another (unless at same location)

**Common traps:**
- Sign errors in work-energy theorem (work done BY system vs. ON system)
- Forgetting to account for rotational kinetic energy in rolling problems
- Confusing instantaneous and average quantities
- In optics: real vs. virtual images have different sign conventions in different textbooks — know which convention is in use

---

## Mathematics

**Core framework**: Formal deductive reasoning — don't rely on intuition for edge cases.

**Key tactics:**
- Read quantifiers carefully: "there exists" vs. "for all" changes everything
- For proofs: consider both directions for iff; consider edge cases (n=0, empty set, boundary conditions)
- Combinatorics: identify whether order matters (permutation) and whether repetition is allowed
- Probability: is this with or without replacement? Conditional or marginal?
- Linear algebra: distinguish rank, nullity, dimension of span, eigenspace dimension

**Common traps:**
- Confusing a function being differentiable vs. continuously differentiable
- "Continuous" does not imply "differentiable" (but differentiable does imply continuous)
- Series convergence: pointwise vs. uniform convergence behave differently
- Matrix multiplication is not commutative — order matters
- For number theory: "prime" vs. "irreducible" are equivalent only in unique factorization domains

---

## Computer Science Theory

**Core framework**: Formal models of computation — reductions, complexity classes, and asymptotic analysis.

**Key tactics:**
- Complexity: Know the standard hierarchy: P ⊆ NP ⊆ PSPACE ⊆ EXPTIME. Containments are known; equalities are open (except PSPACE ≠ EXPTIME)
- Reductions: A ≤p B means "A reduces to B" — if B is easy, so is A. If A is hard, so is B.
- Algorithms: Distinguish best/worst/average case. Big-O is worst-case by convention unless stated otherwise
- Data structures: Amortized complexity is not per-operation worst case
- Automata: DFA and NFA recognize the same class (regular languages); PDA recognizes CFLs; TM recognizes RE

**Common traps:**
- "NP-hard" does not mean "not in P" (if P=NP, NP-hard problems are in P — we just don't know)
- Greedy algorithms are correct for some problems and wrong for others — always verify with a counterexample or exchange argument
- Hash tables: O(1) *expected*, not guaranteed worst case
- Graph algorithms: Dijkstra fails with negative edges; Bellman-Ford handles them but not negative cycles

---

## Economics

**Core framework**: Marginal analysis, equilibrium, and incentive reasoning.

**Key tactics:**
- Microeconomics: Think at the margin. Optimal decisions equate marginal benefit to marginal cost
- Supply/demand: Distinguish shifts IN the curve (from changes in non-price factors) vs. movements ALONG the curve (price changes)
- Game theory: Dominant strategy → Nash equilibrium → Pareto optimality are different concepts
- Macroeconomics: Distinguish short-run vs. long-run (SRAS vs. LRAS shifts differently)
- Elasticity: Elastic demand (|ε| > 1): price increase → total revenue falls. Inelastic: price increase → revenue rises.

**Common traps:**
- Comparative advantage ≠ absolute advantage
- "Deadweight loss" arises from any deviation from competitive equilibrium, not just taxes
- Fiscal vs. monetary policy have different transmission mechanisms and lags
- Ricardian equivalence: government borrowing may not stimulate if consumers anticipate future taxes

---

## Philosophy & Logic

**Core framework**: Formal argument structure and validity vs. soundness.

**Key tactics:**
- An argument is **valid** if the conclusion follows from the premises; **sound** if also the premises are true
- Distinguish deductive vs. inductive vs. abductive arguments
- For ethics questions: identify which framework (consequentialism, deontology, virtue ethics, contractualism) is in play
- Identify the implicit premises in an argument — these are often where the question is testing
- Modal logic: distinguish necessity (□) from possibility (◇); know the S4/S5 axioms if relevant

**Common traps:**
- "Valid" in formal logic ≠ "correct" in everyday speech
- Ad hominem, straw man, begging the question — know the formal definitions, not just the colloquial use
- The is/ought gap (Hume): descriptive premises don't entail normative conclusions without a bridge principle
- Trolley problems: the point is usually the principle that distinguishes the cases, not the answer

---

## History & Political Science

**Core framework**: Causal and contextual reasoning — distinguish proximate from underlying causes.

**Key tactics:**
- For causal questions: what was necessary? What was sufficient? The answer may be neither
- Distinguish historiographical debates from established consensus — benchmark questions tend to test consensus
- Political science: distinguish empirical claims (what is) from normative claims (what ought to be)
- For constitutional/legal questions: know the formal doctrine, not just the intuitive outcome

**Common traps:**
- Teleological reasoning (treating outcomes as inevitable) — history wasn't determined in advance
- Confusing correlation in historical data with causation
- Presentism: applying contemporary values/concepts to historical actors anachronistically
- Conflating the views of individual historical figures with the movements they represented

---

## General Quantitative Reasoning

For any question involving numbers, estimates, or calculations:

1. **Sanity check with orders of magnitude** first
2. **Unit analysis** — if you can't get units to work out, your formula is wrong
3. For estimates: identify what you know, what you need, and chain them together (Fermi estimation)
4. Avoid false precision — if inputs are ±10%, the output can't be more precise than ±10%
5. Check: does the answer change sign or direction at any extreme? Is that physically/logically reasonable?


Application output format: put the final answer only inside <answer></answer> tags at the end of the response. Keep any explanation outside those tags. Follow this output format if a skill below describes a different presentation format. Only the capabilities supplied by the application are available.
\end{Verbatim}
\end{PTwoSSkill}

\begin{PTwoSSkill}{PTwoSSkillBlue}{Mathematical and logical reasoning: AIME}
\begin{Verbatim}[breaklines,breakanywhere,fontsize=\scriptsize]
Application output format: put the final answer only inside <answer></answer> tags at the end of the response. Keep any explanation outside those tags. Follow this output format if a skill below describes a different presentation format. Only the capabilities supplied by the application are available.



# Math and Logic Reasoning Skill

This skill governs how the agent approaches problems that demand **exact, verifiable answers** through structured reasoning. It applies to two broad problem families — **mathematical reasoning** and **logical deduction** — with shared principles and specialized techniques for each.

---

## Core Principles (Apply to Everything)

### 1. Never Skip Steps in Your Head
The most common source of error is "obvious" steps that aren't verified. Write out every nontrivial manipulation. A step that takes 2 seconds to write saves minutes of backtracking.

### 2. State What You're Solving Before You Solve It
Before any computation, articulate:
- What is the unknown? What form should the answer take?
- What are the constraints or given conditions?
- What counts as a valid solution?

This prevents solving the wrong problem.

### 3. Maintain a Clean Separation Between Work and Conclusion
Your scratch work can be messy. Your final answer should be clearly stated, with the key logical chain summarized. Never bury the answer in the middle of working.

### 4. Verify Before Finalizing
Every answer should pass at least one check:
- For math: substitute back, check edge cases, sanity-check the magnitude
- For logic: verify your assignment satisfies *all* constraints, not just the ones you used

### 5. If Stuck, Change Representation
Stuckness usually means the current representation is wrong. Try: different variable naming, a diagram, a small concrete case, the contrapositive, casework, symmetry arguments, or modular arithmetic.

---

## Part I: Mathematical Reasoning

### Problem Intake

Before computing anything, ask:
1. **What type is this?** (algebra, number theory, combinatorics, geometry, calculus, probability)
2. **What is the answer format?** (integer, fraction, expression, "find all X such that...")
3. **Are there hidden constraints?** (integer solutions only? positive reals? non-degenerate triangles?)

Read the problem twice. Misreading is the #1 cause of wasted effort.

### Execution Framework

**Phase 1 — Setup**
- Define variables clearly. Use conventional notation for the domain.
- Write out all given conditions as equations/inequalities.
- Note what you need to find.

**Phase 2 — Strategy Selection**
Choose your approach before executing. Common strategies:

| Situation | Strategy |
|-----------|----------|
| System of equations | Substitution, elimination, or matrix methods |
| Divisibility / modular | Factor, use mod arithmetic, check small cases |
| Counting | Bijection, complementary counting, inclusion-exclusion |
| Geometric | Coordinate geometry, similarity, area ratios, trigonometry |
| Optimization | AM-GM, Cauchy-Schwarz, calculus (if allowed), Lagrange multipliers |
| Recursion / sequences | Find closed form, characteristic equation, generating functions |
| Existence / construction | Explicit construction OR pigeonhole / probabilistic |

**Phase 3 — Execution**
- Work symbolically as long as possible before substituting numbers.
- At each step, ask: *does this simplify? does a factor cancel?*
- Flag any step where you divide (check divisor ≠ 0) or square (check sign).

**Phase 4 — Verification**
- Substitute your answer back into the original equations.
- Check boundary/edge cases.
- Sanity check: is the magnitude reasonable? Does it have the right sign/parity?
- For competition problems with integer answers: double-check arithmetic on the final computation.

### Number Theory Heuristics
- Try small cases first — patterns often emerge for n=1,2,3,4.
- Mod 2, mod 3, mod 9, mod 10 are your first tools for divisibility.
- For Diophantine equations: parametrize solutions, check if infinite families exist.
- GCD/LCM: remember gcd(a,b)·lcm(a,b) = a·b for positive integers.

### Combinatorics Heuristics
- Always ask: are objects distinguishable? Is order relevant?
- Complementary counting is often cleaner than direct counting.
- If a formula seems off by a factor of 2 or k!, you probably over/under-counted permutations.
- Recursion: define f(n) clearly, find base cases, verify your recurrence on small n.

### Algebra / Polynomials
- Factor before expanding whenever possible.
- Vieta's formulas: if you know roots, you know symmetric functions of roots.
- For inequalities: identify equality conditions first — they often suggest substitutions.

### Geometry
- Draw a labeled diagram. Geometry done without a diagram is almost always wrong.
- Mark what's given, what's equal, what's perpendicular.
- Coordinates work well for "find length/area" problems; synthetic is cleaner for "prove."
- Similar triangles: identify them by angle-angle, then set up ratios carefully.

---

## Part II: Logical Deduction

This covers constraint satisfaction, logic grid puzzles, syllogistic chains, and any problem where you must derive what *must* be true given a set of rules.

### The Golden Rule of Logic Puzzles
**Never guess. Never assume. Only record what is forced by the constraints.**

If you can't prove something is true, it isn't known yet.

### Constraint Setup

1. **Inventory**: List all entities, attributes, and possible values.
   - e.g., People: {Alice, Bob, Carol}; Pets: {cat, dog, fish}; Houses: {red, blue, green}

2. **Build an elimination grid** for each (attribute × entity) pair. Start with all values possible.

3. **Translate each clue** into one or more constraint statements:
   - "A is to the left of B" → pos(A) < pos(B)
   - "A and B are neighbors" → |pos(A) - pos(B)| = 1
   - "A has the cat" → pet(A) = cat, which implies pet(B) ≠ cat, pet(C) ≠ cat
   - "The cat owner drinks milk" → pet(X) = cat ↔ drink(X) = milk

### Deduction Loop

Repeat until no new inferences can be drawn:

1. **Direct assignment**: If only one value remains possible for a cell, assign it and propagate (eliminate that value from all other cells in the same row/column).

2. **Forced by elimination**: If a value can only go in one cell in a group, assign it there.

3. **Linked constraints**: Combine two clues to produce a new one.
   - "A is left of B" + "B is left of C" → A is left of C

4. **Positional arithmetic**: For ordered arrangements, use inequalities and count remaining slots.

5. **Case analysis** (last resort): If stuck, branch on the most constrained unknown. Mark it clearly, follow to contradiction or resolution, then try the other branch.

### Common Errors to Avoid

| Error | Prevention |
|-------|-----------|
| Assuming a clue implies its converse | Only conclude what the clue directly states |
| Forgetting a constraint exists | After each deduction, re-check *all* clues for new implications |
| Conflating "neighbor" with "immediately left/right" | Re-read the problem's definition |
| Assigning before confirming uniqueness | Count remaining options before assigning |

### Verification
Once a complete assignment is found:
- Go through **every single clue** and verify the assignment satisfies it.
- Note which clue each verification covers — if any clue is never checked, you might have missed it.

---

## Part III: Presenting Solutions

### For Math Problems
```
[Brief restatement of what's being solved]

**Setup:**
[Define variables, state given conditions]

**Solution:**
[Step-by-step work with clear transitions]

**Answer: [X]**

**Verification:**
[Substitute back or check the answer]
```

### For Logic Puzzles
```
[List of entities and attributes]

**Deduction trace:**
[Numbered steps, each citing the clue used]
  Step 1: From clue 3, ...
  Step 2: Since step 1 forces ..., combined with clue 7, ...

**Solution:**
[Final assignment table]

**Verification:**
[Every clue checked: ✓ Clue 1: ..., ✓ Clue 2: ...]
```

---

## When to Use the Reference Files

- See `references/competition-techniques.md` for deeper treatment of specific math techniques (generating functions, modular inverses, projective geometry, etc.)
- See `references/logic-patterns.md` for a catalog of common constraint types and their propagation rules

---

## Meta-Note on Difficulty

Hard problems (AIME #13-15, multi-constraint logic grids) often require:
- **Insight**, not just procedure: a clever substitution, a symmetry argument, or a reframe of the problem
- **Patience**: the solution may require 5+ deduction steps before progress becomes visible
- **Willingness to restart**: if you've been solving for a while and the answer looks messy, consider whether you chose the right approach at the outset

Elegance is a signal of correctness. If the answer is ugly, look for a cleaner path.




# Bundled reference: references/competition-techniques.md

# Competition Math Techniques Reference

## Table of Contents
1. Modular Arithmetic & Number Theory
2. Generating Functions
3. Combinatorial Identities
4. Inequalities
5. Geometric Tools
6. Algebraic Manipulation
7. Probabilistic / Expectation

---

## 1. Modular Arithmetic & Number Theory

### Modular Inverses
- a⁻¹ mod p exists iff gcd(a, p) = 1
- Compute via Fermat's little theorem: a⁻¹ ≡ a^(p-2) mod p (p prime)
- Or extended Euclidean algorithm for non-prime moduli

### Chinese Remainder Theorem
- If n = p₁p₂...pₖ with pᵢ pairwise coprime, then ℤ/nℤ ≅ ℤ/p₁ℤ × ... × ℤ/pₖℤ
- Use to reconstruct solutions from residues modulo each prime factor

### Lifting the Exponent (LTE)
For odd prime p, p | a-b but p ∤ a, p ∤ b:
- νₚ(aⁿ - bⁿ) = νₚ(a-b) + νₚ(n)
For p = 2, a, b odd:
- ν₂(aⁿ - bⁿ) = ν₂(a-b) + ν₂(a+b) + ν₂(n) - 1

### Zsygmondy / Bang's Theorem
aⁿ - bⁿ has a prime factor not dividing aᵏ - bᵏ for all k < n (with small exceptions). Useful for existence proofs.

### Legendre's Formula
The largest power of prime p dividing n! is: Σ⌊n/pᵏ⌋ for k = 1, 2, ...

### Orders and Primitive Roots
- ord_n(a) = smallest positive k with aᵏ ≡ 1 mod n
- For prime p, primitive roots exist; there are φ(p-1) of them
- A number a is a quadratic residue mod p iff a^((p-1)/2) ≡ 1 mod p (Euler's criterion)

---

## 2. Generating Functions

### Basic Setup
To find the number of ways to choose items with given constraints, encode options as polynomials and multiply.

**Example**: Coins with values 1, 5, 10, 25. Number of ways to make n cents = coefficient of xⁿ in:
(1 + x + x² + ...) · (1 + x⁵ + x¹⁰ + ...) · ...

### Exponential Generating Functions (EGF)
For labeled structures, use EGF: f(x) = Σ aₙ xⁿ/n!
- EGF of derangements: e^(-x) / (1-x)
- Labeled structures often split as products of EGFs

### Recurrence to Closed Form
Given aₙ = c₁aₙ₋₁ + c₂aₙ₋₂:
1. Characteristic equation: r² = c₁r + c₂
2. Roots r₁, r₂ → general solution aₙ = Ar₁ⁿ + Br₂ⁿ
3. Use initial conditions to find A, B

---

## 3. Combinatorial Identities

### Vandermonde's Identity
Σₖ C(m,k)C(n,r-k) = C(m+n, r)

### Hockey Stick
Σₖ₌₀ⁿ C(k+r, r) = C(n+r+1, r+1)

### Stars and Bars
Number of non-negative integer solutions to x₁+...+xₖ = n is C(n+k-1, k-1).
With each xᵢ ≥ 1: C(n-1, k-1).

### Inclusion-Exclusion
|A₁ ∪ ... ∪ Aₙ| = Σ|Aᵢ| - Σ|Aᵢ∩Aⱼ| + Σ|Aᵢ∩Aⱼ∩Aₖ| - ...

### Burnside's Lemma (for counting under symmetry)
Number of distinct objects = (1/|G|) Σ_{g∈G} |Fix(g)|
where Fix(g) = set of objects fixed by symmetry g.

---

## 4. Inequalities

### AM-GM
(a₁+...+aₙ)/n ≥ (a₁...aₙ)^(1/n), equality iff all aᵢ equal.
**Trick**: To find when equality holds, set all terms equal. This guides substitutions.

### Cauchy-Schwarz
(Σaᵢbᵢ)² ≤ (Σaᵢ²)(Σbᵢ²)
In Engel/Titu form: Σ aᵢ²/bᵢ ≥ (Σaᵢ)²/Σbᵢ

### Power Mean Inequality
M_r ≥ M_s for r > s, where M_r = ((Σxᵢʳ)/n)^(1/r)

### Rearrangement Inequality
If a₁≤...≤aₙ and b₁≤...≤bₙ, then Σaᵢb_{σ(i)} is maximized when σ is the identity and minimized when σ reverses the order.

### Jensen's Inequality
For convex f: f((Σwᵢxᵢ)/Σwᵢ) ≤ (Σwᵢf(xᵢ))/Σwᵢ
Flip for concave f.

---

## 5. Geometric Tools

### Coordinate Geometry Checklist
- Set up coordinates to exploit symmetry (place midpoint at origin, axis along line of symmetry)
- Shoelace formula for polygon area: A = ½|Σ(xᵢyᵢ₊₁ - xᵢ₊₁yᵢ)|
- Distance from point (x₀,y₀) to line ax+by+c=0: |ax₀+by₀+c|/√(a²+b²)

### Circle Theorems
- Power of a point: PA·PB = PC·PD (for two chords through P)
- Extended law of sines: a/sin A = 2R
- Ptolemy's theorem: AC·BD = AB·CD + AD·BC (cyclic quadrilateral)

### Area Methods
- [ABC] = ½ ab sin C = ½ · base · height = √(s(s-a)(s-b)(s-c)) (Heron's)
- Ratio of areas = ratio of bases (same height) or ratio of products of sides (shared angle)

### Trigonometric Identities for Competition
- sin(A+B) = sinA cosB + cosA sinB
- cos(2A) = 1 - 2sin²A = 2cos²A - 1
- Product to sum: 2sinA cosB = sin(A+B) + sin(A-B)

### Vectors
- Centroid G = (A+B+C)/3
- Midpoint M = (A+B)/2
- Line through points A, B: P = A + t(B-A)
- Cevian concurrence: Ceva's theorem (trigonometric form often cleanest)

---

## 6. Algebraic Manipulation

### Factoring Patterns
- a²-b² = (a-b)(a+b)
- a³±b³ = (a±b)(a²∓ab+b²)
- aⁿ-bⁿ = (a-b)(aⁿ⁻¹+aⁿ⁻²b+...+bⁿ⁻¹)
- Sophie Germain: a⁴+4b⁴ = (a²+2b²+2ab)(a²+2b²-2ab)

### Substitution Strategies
- Symmetric expressions in a,b: try s=a+b, p=ab
- Trigonometric substitution for √(1-x²): x = sin θ
- Hyperbolic for √(x²-1): x = cosh t
- For constraints like ab+bc+ca=1: try a=tan A, b=tan B, c=tan C with A+B+C=π/2

### Polynomial Root Tricks
- Vieta: for xⁿ+pₙ₋₁xⁿ⁻¹+...+p₀, sum of roots = -pₙ₋₁, product of roots = (-1)ⁿp₀
- If r is a root of p(x), then (x-r) | p(x)
- Rational root theorem: rational roots of integer polynomial are ±(factor of constant)/(factor of leading)

---

## 7. Probability and Expectation

### Key Formulas
- E[X] = Σ x·P(X=x) (discrete) or ∫ x·f(x)dx (continuous)
- Linearity of expectation: E[X+Y] = E[X]+E[Y] (always, even if dependent)
- Variance: Var(X) = E[X²] - (E[X])²

### Indicator Variables
For complex counting problems, define Iₖ = 1 if event k occurs, 0 otherwise.
E[Σ Iₖ] = Σ P(event k). Often dramatically simplifies calculations.

### Geometric Probability
P(hit region A) = Area(A)/Area(total). Be careful about the sample space definition.

### Conditional Probability
P(A|B) = P(A∩B)/P(B). Bayes: P(A|B) = P(B|A)P(A)/P(B).
For sequential problems, draw a tree and multiply along branches.

### Markov Chains / States
Define states, write transition matrix, solve system of equations for absorption probabilities or expected hitting times.




# Bundled reference: references/logic-patterns.md

# Logic Puzzle Patterns Reference

## Table of Contents
1. Constraint Types and How to Propagate Them
2. Grid Setup for Einstein/Zebra Puzzles
3. Syllogistic Reasoning
4. Common Trap Patterns
5. Worked Strategy Templates

---

## 1. Constraint Types and Propagation Rules

### Direct Assignment Constraints
**Form**: "X has attribute A" or "The person who Xs is Y"
**Propagation**: 
- Assign value immediately
- Eliminate that value from all other entities in the same attribute category

### Negative Constraints
**Form**: "X does not have attribute A" or "X is not Y"
**Propagation**: 
- Eliminate A from X's possibilities
- If only one possibility remains for X in that attribute, assign it

### Relational Constraints (Ordered Positions)
**Form**: "X is immediately to the left of Y", "X is somewhere left of Y", "X and Y are neighbors"

| Constraint | Propagation |
|-----------|-------------|
| pos(X) = pos(Y) + 1 (X right of Y) | Eliminate position 1 for X, last position for Y |
| \|pos(X) - pos(Y)\| = 1 (neighbors) | Neither can be isolated on opposite ends unless adjacent |
| pos(X) < pos(Y) (X left of Y) | Eliminate rightmost positions for X, leftmost for Y |
| pos(X) = n (direct position) | Remove X from all other positions |

### Biconditional / Paired Constraints
**Form**: "The X-owner is also the Y-owner" or "A and B have the same attribute"
**Propagation**:
- Any elimination in one propagates to the other
- E.g., "The milk drinker lives in house 3" → if we know who lives in house 3, they drink milk; if we know who drinks milk, they live in house 3

### Conditional Constraints
**Form**: "If X has A, then X has B"
**Propagation**:
- If X is assigned A → assign B to X
- If X cannot have B (B is eliminated) → eliminate A from X
- **Do NOT conclude**: "X does not have A, therefore X does not have B" (that's the inverse fallacy)

### "Same/Different" Constraints
**Form**: "X and Y have the same pet" or "No two people share the same job"
- Same: treat the two as a linked pair for that attribute
- Different: this is often already implicit in the grid structure (injective assignment)

---

## 2. Grid Setup for Einstein-Style Puzzles

### Standard Setup
Entities: n people/houses with k attributes each. Total cells: n × k.

```
      | Attr1 | Attr2 | Attr3 | ... |
------+-------+-------+-------+-----|
Pos 1 |       |       |       |     |
Pos 2 |       |       |       |     |
...   |       |       |       |     |
Pos n |       |       |       |     |
```

### Working the Grid
1. Apply all direct-assignment clues first. These give the most information for free.
2. Apply all negative clues next.
3. Then work relational clues using elimination.
4. Re-scan all clues after every new assignment — new assignments often unlock clues that were previously unhelpful.

### Position Constraints for 5-House Puzzles
- "House 1" = leftmost. "Middle house" = house 3.
- "Neighbor" always means adjacent (±1 position), not in same cluster.
- When a clue says "next to," it's symmetric: both orderings (A-B and B-A) are valid.

---

## 3. Syllogistic and Multi-Step Deduction

### Chain Reasoning
Link individual clues into chains to produce new facts:
```
Clue A: X → Y (if X then Y)
Clue B: Y → Z (if Y then Z)
∴ X → Z
```

### Disjunctive Syllogism
```
X or Y (one of these must be true)
not X (eliminate one)
∴ Y
```

### Hypothetical Elimination
When stuck: assume one branch → derive → check for contradiction.
```
Suppose attr(Person A) = Value V
→ From clue 3, attr(Person B) = Value W
→ From clue 7, attr(Person B) ≠ Value W
→ Contradiction. Therefore attr(Person A) ≠ Value V.
```
Always document which assumption you're under when doing this. Abandon cleanly if no contradiction is found (try the other value).

---

## 4. Common Trap Patterns

### The Converse Fallacy
**Trap**: "The Swede has a dog" → "The dog owner is Swedish"
Both follow from the same statement (it's biconditional), so this is actually fine.

But: "The Swede owns a dog" does NOT mean "Everyone with a dog is Swedish."

In most well-formed puzzles, attribute assignments are injective (one-to-one). So "X has attribute A" does imply "no one else has attribute A." In ambiguous puzzles, check whether attributes are exclusive before propagating this way.

### Confusing "Or" Types
- Inclusive or: A or B or both (default in logic)
- Exclusive or: A or B but not both
Competition logic puzzles usually deal with injective assignments (exclusive by structure), but word problems may use inclusive or.

### Forgetting Symmetry of Neighbor Clues
"A is next to B" means A could be left of B or right of B. Don't prematurely commit to one order without justification.

### Stacking Without Grounding
Chaining 4 inferences and then discovering one was based on an unverified assumption. Fix: number your steps, and label each with the clue(s) it uses.

---

## 5. Strategy Templates

### Template: "Who can occupy position X?"
1. List all constraints involving position X.
2. For each entity, check whether it violates any constraint if placed in X.
3. Assign the unique entity that doesn't violate anything.

### Template: "What value must attribute A of entity X be?"
1. Start with all possible values for A.
2. For each constraint involving X or attribute A, eliminate impossible values.
3. Check if remaining values = 1. If yes, assign.

### Template: "Which two entities are linked?"
When a clue connects two attributes (e.g., "milk drinker lives in house 3"):
1. If you know the value of one attribute, you know the entity for both.
2. Track this as a "soft link" until one side is resolved.

### Template: "Forced by position constraints"
For ordered constraints like pos(A) < pos(B) < pos(C) in a 5-house puzzle:
1. The gap forces: A can be at most position 3, C can be at minimum position 3.
2. This often pins one of them, which cascades.

---

## Quick Reference: What to Try When Stuck

| Situation | Action |
|-----------|--------|
| No obvious next step | Re-read every clue once; often one is underused |
| One entity has only one possible value remaining | Assign it immediately |
| One value can only go to one entity in a group | Assign it there |
| Two clues seem to interact | Try chaining them explicitly |
| Still stuck after above | Branch on most constrained unknown, look for contradiction |
| Contradiction found | Backtrack and flip the branching assumption |
| No contradiction found in branch | The branch must be explored further; try second-level branching |


Application output format: put the final answer only inside <answer></answer> tags at the end of the response. Keep any explanation outside those tags. Follow this output format if a skill below describes a different presentation format. Only the capabilities supplied by the application are available.
\end{Verbatim}
\end{PTwoSSkill}

\begin{PTwoSSkill}{PTwoSSkillBlue}{Spreadsheet manipulation: SpreadsheetBench}
\begin{Verbatim}[breaklines,breakanywhere,fontsize=\scriptsize]
# XLSX Manipulation Skill

## Overview

This skill enables programmatic creation, editing, and manipulation of Microsoft Excel (.xlsx) spreadsheets using the **openpyxl** library. Create professional spreadsheets with formulas, formatting, charts, and data validation without manual editing.

## How to Use

1. Describe the spreadsheet you want to create or modify
2. Provide data, formulas, or formatting requirements
3. I'll generate openpyxl code and execute it

**Example prompts:**
- "Create a budget spreadsheet with monthly tracking"
- "Add conditional formatting to highlight values above threshold"
- "Generate a pivot-table-like summary from this data"
- "Create a dashboard with charts and KPIs"

## Domain Knowledge

### openpyxl Fundamentals

```python
from openpyxl import Workbook, load_workbook
from openpyxl.styles import Font, Fill, Border, Alignment
from openpyxl.chart import BarChart, Reference

# Create new workbook
wb = Workbook()
ws = wb.active

# Or open existing
wb = load_workbook('existing.xlsx')
ws = wb.active
```

### Workbook Structure
```
Workbook
├── worksheets (sheets/tabs)
│   ├── cells (data storage)
│   ├── rows/columns (formatting)
│   ├── merged_cells
│   └── charts
├── defined_names (named ranges)
└── styles (formatting templates)
```

### Working with Cells

#### Basic Cell Operations
```python
# By cell reference
ws['A1'] = 'Header'
ws['B1'] = 42

# By row, column
ws.cell(row=1, column=3, value='Data')

# Multiple cells
ws['A1:C1'] = [['Col1', 'Col2', 'Col3']]

# Append rows
ws.append(['Row', 'Data', 'Here'])
```

#### Reading Cells
```python
# Single cell
value = ws['A1'].value

# Cell range
for row in ws['A1:C3']:
    for cell in row:
        print(cell.value)

# Iterate rows
for row in ws.iter_rows(min_row=1, max_row=10, min_col=1, max_col=3):
    for cell in row:
        print(cell.value)
```

### Formulas
```python
# Basic formulas
ws['D1'] = '=SUM(A1:C1)'
ws['D2'] = '=AVERAGE(A2:C2)'
ws['E1'] = '=IF(D1>100,"High","Low")'

# Named ranges
from openpyxl.workbook.defined_name import DefinedName
ref = "Sheet!$A$1:$C$10"
defn = DefinedName("SalesData", attr_text=ref)
wb.defined_names.add(defn)

# Use named range
ws['F1'] = '=SUM(SalesData)'
```

### Formatting

#### Cell Styles
```python
from openpyxl.styles import Font, Fill, PatternFill, Border, Side, Alignment

# Font
ws['A1'].font = Font(
    name='Arial',
    size=14,
    bold=True,
    italic=False,
    color='FF0000'  # Red
)

# Fill (background)
ws['A1'].fill = PatternFill(
    start_color='FFFF00',  # Yellow
    end_color='FFFF00',
    fill_type='solid'
)

# Border
thin_border = Border(
    left=Side(style='thin'),
    right=Side(style='thin'),
    top=Side(style='thin'),
    bottom=Side(style='thin')
)
ws['A1'].border = thin_border

# Alignment
ws['A1'].alignment = Alignment(
    horizontal='center',
    vertical='center',
    wrap_text=True
)
```

#### Number Formats
```python
# Currency
ws['B2'].number_format = '$#,##0.00'

# Percentage
ws['C2'].number_format = '0.00%'

# Date
ws['D2'].number_format = 'YYYY-MM-DD'

# Custom
ws['E2'].number_format = '#,##0.00 "units"'
```

#### Conditional Formatting
```python
from openpyxl.formatting.rule import ColorScaleRule, CellIsRule, FormulaRule
from openpyxl.styles import PatternFill

# Color scale (heatmap)
color_scale = ColorScaleRule(
    start_type='min', start_color='FF0000',
    end_type='max', end_color='00FF00'
)
ws.conditional_formatting.add('A1:A10', color_scale)

# Cell value rule
red_fill = PatternFill(start_color='FFCCCC', end_color='FFCCCC', fill_type='solid')
rule = CellIsRule(operator='greaterThan', formula=['100'], fill=red_fill)
ws.conditional_formatting.add('B1:B10', rule)
```

### Charts
```python
from openpyxl.chart import BarChart, LineChart, PieChart, Reference

# Prepare data
data = Reference(ws, min_col=2, min_row=1, max_col=3, max_row=5)
categories = Reference(ws, min_col=1, min_row=2, max_row=5)

# Bar Chart
chart = BarChart()
chart.type = "col"  # or "bar" for horizontal
chart.title = "Sales by Region"
chart.add_data(data, titles_from_data=True)
chart.set_categories(categories)
chart.shape = 4
ws.add_chart(chart, "E1")

# Line Chart
line = LineChart()
line.title = "Trend Analysis"
line.add_data(data, titles_from_data=True)
line.set_categories(categories)
ws.add_chart(line, "E15")

# Pie Chart
pie = PieChart()
pie.add_data(data, titles_from_data=True)
pie.set_categories(categories)
ws.add_chart(pie, "M1")
```

### Data Validation
```python
from openpyxl.worksheet.datavalidation import DataValidation

# Dropdown list
dv = DataValidation(
    type="list",
    formula1='"Option1,Option2,Option3"',
    allow_blank=True
)
dv.error = "Please select from list"
dv.errorTitle = "Invalid Input"
ws.add_data_validation(dv)
dv.add('A1:A100')

# Number range
dv_num = DataValidation(
    type="whole",
    operator="between",
    formula1="1",
    formula2="100"
)
ws.add_data_validation(dv_num)
dv_num.add('B1:B100')
```

### Sheet Operations
```python
# Create new sheet
ws2 = wb.create_sheet("Data")
ws3 = wb.create_sheet("Summary", 0)  # At position 0

# Rename
ws.title = "Main Report"

# Delete
del wb["Sheet2"]

# Copy
source = wb["Template"]
target = wb.copy_worksheet(source)
```

### Row/Column Operations
```python
# Set column width
ws.column_dimensions['A'].width = 20

# Set row height
ws.row_dimensions[1].height = 30

# Hide column
ws.column_dimensions['C'].hidden = True

# Freeze panes
ws.freeze_panes = 'B2'  # Freeze row 1 and column A

# Auto-filter
ws.auto_filter.ref = "A1:D100"
```

## Best Practices

1. **Use Templates**: Start with a .xlsx template for complex formatting
2. **Batch Operations**: Minimize cell-by-cell operations for speed
3. **Named Ranges**: Use defined names for clearer formulas
4. **Data Validation**: Add validation to prevent input errors
5. **Save Incrementally**: For large files, save periodically

## Common Patterns

### Data Import
```python
def import_csv_to_xlsx(csv_path, xlsx_path):
    import csv
    wb = Workbook()
    ws = wb.active
    
    with open(csv_path) as f:
        reader = csv.reader(f)
        for row in reader:
            ws.append(row)
    
    wb.save(xlsx_path)
```

### Report Template
```python
def create_monthly_report(data, output_path):
    wb = Workbook()
    ws = wb.active
    ws.title = "Monthly Report"
    
    # Headers
    headers = ['Date', 'Revenue', 'Expenses', 'Profit']
    ws.append(headers)
    
    # Style headers
    for col in range(1, 5):
        cell = ws.cell(1, col)
        cell.font = Font(bold=True)
        cell.fill = PatternFill('solid', fgColor='4472C4')
        cell.font = Font(bold=True, color='FFFFFF')
    
    # Data
    for row in data:
        ws.append(row)
    
    # Add totals
    last_row = len(data) + 1
    ws.cell(last_row + 1, 1, 'TOTAL')
    ws.cell(last_row + 1, 2, f'=SUM(B2:B{last_row})')
    ws.cell(last_row + 1, 3, f'=SUM(C2:C{last_row})')
    ws.cell(last_row + 1, 4, f'=SUM(D2:D{last_row})')
    
    wb.save(output_path)
```

## Examples

### Example 1: Budget Tracker
```python
from openpyxl import Workbook
from openpyxl.styles import Font, PatternFill, Alignment, Border, Side
from openpyxl.utils import get_column_letter

wb = Workbook()
ws = wb.active
ws.title = "Budget 2024"

# Headers
months = ['Category', 'Jan', 'Feb', 'Mar', 'Q1 Total']
ws.append(months)

# Categories and data
budget_data = [
    ['Salary', 5000, 5000, 5000],
    ['Rent', -1500, -1500, -1500],
    ['Utilities', -200, -180, -220],
    ['Food', -400, -450, -380],
    ['Transport', -150, -160, -140],
    ['Entertainment', -200, -250, -200],
]

for row in budget_data:
    ws.append(row + [f'=SUM(B{ws.max_row + 1}:D{ws.max_row + 1})'])

# Total row
ws.append(['TOTAL', 
    f'=SUM(B2:B{ws.max_row})',
    f'=SUM(C2:C{ws.max_row})',
    f'=SUM(D2:D{ws.max_row})',
    f'=SUM(E2:E{ws.max_row})'
])

# Formatting
header_fill = PatternFill('solid', fgColor='366092')
header_font = Font(bold=True, color='FFFFFF')

for cell in ws[1]:
    cell.fill = header_fill
    cell.font = header_font
    cell.alignment = Alignment(horizontal='center')

# Currency format
for row in ws.iter_rows(min_row=2, min_col=2, max_col=5):
    for cell in row:
        cell.number_format = '$#,##0.00'

# Column widths
ws.column_dimensions['A'].width = 15
for col in range(2, 6):
    ws.column_dimensions[get_column_letter(col)].width = 12

wb.save('budget_2024.xlsx')
```

### Example 2: Sales Dashboard
```python
from openpyxl import Workbook
from openpyxl.chart import BarChart, PieChart, Reference
from openpyxl.styles import Font, PatternFill

wb = Workbook()
ws = wb.active
ws.title = "Sales Dashboard"

# Data
ws.append(['Region', 'Q1', 'Q2', 'Q3', 'Q4'])
data = [
    ['North', 150000, 165000, 180000, 195000],
    ['South', 120000, 125000, 140000, 155000],
    ['East', 180000, 190000, 210000, 225000],
    ['West', 95000, 110000, 125000, 140000],
]
for row in data:
    ws.append(row)

# Bar Chart
data_ref = Reference(ws, min_col=2, min_row=1, max_col=5, max_row=5)
cats_ref = Reference(ws, min_col=1, min_row=2, max_row=5)

bar = BarChart()
bar.type = "col"
bar.title = "Quarterly Sales by Region"
bar.add_data(data_ref, titles_from_data=True)
bar.set_categories(cats_ref)
bar.height = 10
bar.width = 15
ws.add_chart(bar, "A8")

# Pie Chart - Q4 breakdown
pie_data = Reference(ws, min_col=5, min_row=1, max_row=5)
pie = PieChart()
pie.title = "Q4 Market Share"
pie.add_data(pie_data, titles_from_data=True)
pie.set_categories(cats_ref)
ws.add_chart(pie, "J8")

wb.save('sales_dashboard.xlsx')
```

## Limitations

- Cannot execute VBA macros
- Complex pivot tables not fully supported
- Limited sparkline support
- External data connections not supported
- Some advanced chart types unavailable

## Installation

```bash
pip install openpyxl
```

## Resources

- [openpyxl Documentation](https://openpyxl.readthedocs.io/)
- [GitHub Repository](https://github.com/theorchard/openpyxl)
- [Working with Styles](https://openpyxl.readthedocs.io/en/stable/styles.html)
\end{Verbatim}
\end{PTwoSSkill}

The spreadsheet source is distributed under the following license.
\begin{PTwoSSkill}{PTwoSSkillBlue}{Spreadsheet skill: license notice}
\begin{Verbatim}[breaklines,breakanywhere,fontsize=\scriptsize]
MIT License

Copyright (c) 2026 Claude Office Skills Contributors

Permission is hereby granted, free of charge, to any person obtaining a copy
of this software and associated documentation files (the "Software"), to deal
in the Software without restriction, including without limitation the rights
to use, copy, modify, merge, publish, distribute, sublicense, and/or sell
copies of the Software, and to permit persons to whom the Software is
furnished to do so, subject to the following conditions:

The above copyright notice and this permission notice shall be included in all
copies or substantial portions of the Software.

THE SOFTWARE IS PROVIDED "AS IS", WITHOUT WARRANTY OF ANY KIND, EXPRESS OR
IMPLIED, INCLUDING BUT NOT LIMITED TO THE WARRANTIES OF MERCHANTABILITY,
FITNESS FOR A PARTICULAR PURPOSE AND NONINFRINGEMENT. IN NO EVENT SHALL THE
AUTHORS OR COPYRIGHT HOLDERS BE LIABLE FOR ANY CLAIM, DAMAGES OR OTHER
LIABILITY, WHETHER IN AN ACTION OF CONTRACT, TORT OR OTHERWISE, ARISING FROM,
OUT OF OR IN CONNECTION WITH THE SOFTWARE OR THE USE OR OTHER DEALINGS IN THE
SOFTWARE.
\end{Verbatim}
\end{PTwoSSkill}

\subsection{Final Selected P2S Skills}
\label{app:final_skills}
We present one final selected skill for each model--task pair reported
in the main results, using the first completed run throughout.
The listings cover all five target models and all four tasks; the
Llama-3.2-1B--SpreadsheetBench pair is not reported in the main table.
The listings reproduce the artifacts used for the reported P2S evaluations,
including retained initial skills and data-informed spreadsheet drafts.
As above, the boxes reproduce the instruction bodies supplied to the
target models, with packaging metadata omitted and line endings normalized.

\subsubsection{SearchQA}
\label{app:final_skills_searchqa}

\begin{PTwoSSkill}{PTwoSSkillGreen}{SearchQA: Qwen3-8B}
\begin{Verbatim}[breaklines,breakanywhere,fontsize=\scriptsize]
<answer>Approach each clue by first identifying the required answer type—whether it's a decade, exact year, common name, or formal variant—and use all constraint words to narrow possibilities. Avoid distractions from famous but loosely connected entities and prioritize the standard, common-name form over technical or alternate variations. Ensure precision by including necessary articles or descriptors that are part of the canonical answer. When context or embedded media is present, focus on the core concept implied by the surrounding text rather than literal details.</answer>
\end{Verbatim}
\end{PTwoSSkill}

\begin{PTwoSSkill}{PTwoSSkillGreen}{SearchQA: Qwen3-32B}
\begin{Verbatim}[breaklines,breakanywhere,fontsize=\scriptsize]
<answer>Before responding to a Jeopardy! clue, methodically dissect the full wording to identify the exact semantic target, including person, place, thing, or concept, and account for all constraints—wordplay, style, and structure. Avoid jumping to conclusions based on a single keyword; instead, ensure every part of the clue supports the answer. Verify that the proposed response fully explains any pun, quotation, or double meaning. Pay strict attention to form: use the correct phrasing, titles, articles, and ordinals as expected on the show. Favor precise recall over approximate guesses, especially in history and pop culture, and rely on distinctive identifiers to avoid plausible but incorrect substitutions. Finally, validate the response against all elements of the clue—semantic type, context, wording, and date—before committing to the answer.</answer>
\end{Verbatim}
\end{PTwoSSkill}

\begin{PTwoSSkill}{PTwoSSkillGreen}{SearchQA: Llama-3.2-1B}
\begin{Verbatim}[breaklines,breakanywhere,fontsize=\scriptsize]
<deliberate "cue the correct entity" process for answering prompts, using core retrieval cues like title, phrase, or specific details; verify answer semantic fit with all provided information; avoid generic or unrelated guesses; and strictly adhere to clause structure and formatting guidelines.</deliberate>
\end{Verbatim}
\end{PTwoSSkill}

\begin{PTwoSSkill}{PTwoSSkillGreen}{SearchQA: Claude Haiku 4.5}
\begin{Verbatim}[breaklines,breakanywhere,fontsize=\scriptsize]
Answer the Jeopardy! question. Reply with only the answer, inside <answer></answer> tags.
\end{Verbatim}
\end{PTwoSSkill}

\begin{PTwoSSkill}{PTwoSSkillGreen}{SearchQA: GPT-5.5}
\begin{Verbatim}[breaklines,breakanywhere,fontsize=\scriptsize]
Answer the Jeopardy! clue. Reply only with the final response inside <answer></answer> tags; do not include “What is/Who is”.

Process:
1. Treat the clue as asking for Jeopardy’s expected response, not an explanation.
2. Infer the requested response type from wording, syntax, and category-like signals: person, title, place, institution, object, term, number, group, action, time, shared name, etc.
3. If the clue is a fragment, list, caption, quotation, or media reference, infer the missing task from context; never answer blank just because the clue is not a full sentence.
4. For quotations, decide whether the clue asks for speaker, character, work, creator, or quoted words; answer at the level implied by the phrasing.
5. For inaccessible media, use visible text, captions, descriptive wording, era cues, and clue wording to infer the expected response.
6. Recall several plausible candidates, then choose the one satisfying every stated constraint: dates, relationships, titles, definitions, locations, word order, and wordplay.
7. Match exact granularity: not broader, narrower, more explanatory, or merely related; include the generic noun/class if it is needed to make the named answer complete.
8. If the clue asks for a shared name of entities, give the shared name in the form normally used for those entities, with the entity type included when natural or clarifying.
9. Prefer the shortest standard Jeopardy-style response that fully answers the clue; do not add descriptive adjectives unless part of the accepted name.
10. Use common English-language spelling and transliteration as found in general U.S. reference sources; avoid unnecessary special diacritics.
11. Include natural articles for singular count nouns when ordinary speech requires them, especially with “a word/term/name for” clues.
12. Include “the” for names, groups, works, objects, or institutions conventionally said with “the.”
13. Include possessives such as “your” when the clue’s natural response is a body part or personal-reference phrase.
14. Preserve necessary disambiguation with a role, title, parenthetical, or generic noun when a bare name would be unclear or incomplete.
15. For plurals, match what the clue defines: distinguish plural agents/items from a singular action, event, term, or example.
16. For clues giving multiple items, identify whether the answer is the shared rule, the missing/odd item, or one listed item singled out by the category-style context.
17. For wordplay, solve both the literal definition and the language trick; output the target answer, not the explanation.
18. If the clue asks for a time, choose the intended granularity: day, month, year, decade, century, era, or duration.
19. If a standalone proper noun appears, consider associated larger entity, origin, creator, member, classification, or counterpart before answering.
20. Use conventional title/name wording; keep ampersands or exact wording only when part of the familiar response.
21. Do not include alternatives, hedging, explanations, citations, or repeated clue wording.
22. Final output format must be exactly: <answer>final answer</answer>
\end{Verbatim}
\end{PTwoSSkill}

\subsubsection{SQuAD}
\label{app:final_skills_squad}

\begin{PTwoSSkill}{PTwoSSkillGreen}{SQuAD: Qwen3-8B}
\begin{Verbatim}[breaklines,breakanywhere,fontsize=\scriptsize]
You are a passage-based question-answering model.

Process:
1. Read the entire passage/support and the question carefully before answering.
2. Answer using only information explicitly stated or directly implied by the passage.
3. Locate the sentence or phrase in the passage that most directly answers the question.
4. Prefer an exact text span from the passage over a paraphrase.
5. Choose the smallest answer span that fully satisfies the question.
6. Match the expected granularity of the question:
   - If the question asks for a term, give just the term.
   - If it asks for a type/kind/form, give just the distinguishing word or phrase, not the repeated category word unless it is part of the term.
   - If it asks for a value, condition, place, time, or property, give only that value/condition/place/time/property.
7. Do not include extra explanation, examples, surrounding sentence text, or multiple alternatives unless the question explicitly asks for them.
8. Do not add broader categories when the passage gives a more specific answer.
9. Do not add narrower details when the question asks for a general term.
10. Preserve important wording from the passage, including singular/plural form when it matters.
11. If two candidate answers seem possible, choose the one that most directly matches the wording and grammar of the question.
12. Before finalizing, check that the answer:
    - is supported by the passage,
    - answers exactly what was asked,
    - is not overly broad or overly specific,
    - contains no unnecessary repeated words from the question.
13. Put only the final answer inside <answer></answer> tags.
\end{Verbatim}
\end{PTwoSSkill}

\begin{PTwoSSkill}{PTwoSSkillGreen}{SQuAD: Qwen3-32B}
\begin{Verbatim}[breaklines,breakanywhere,fontsize=\scriptsize]
You are a question-answering model. Read the passage/support carefully and answer the question using only information from it.

Process:
1. Identify exactly what the question is asking for by turning it into a blank statement.
2. Find the sentence or phrase in the support that fills that blank.
3. Choose the shortest complete answer that satisfies the question.
4. Do not repeat the question, restate the whole sentence, or include extra explanation.
5. Prefer wording from the support, but adjust capitalization, number, or form only when needed to fit the question grammatically.
6. If the support gives a longer phrase, remove unnecessary surrounding words, articles, or repeated category words already supplied by the question.
7. If the question asks for a term/name, give only the term/name.
8. If the question asks what happens or what something does, give only the concise action, result, or property.
9. If multiple nearby words could answer, pick the one that most directly fills the blank and matches all constraints in the question.
10. Before finalizing, check that the answer is not longer than necessary and is directly supported by the passage.

Output only the final answer inside <answer></answer> tags.
\end{Verbatim}
\end{PTwoSSkill}

\begin{PTwoSSkill}{PTwoSSkillGreen}{SQuAD: Llama-3.2-1B}
\begin{Verbatim}[breaklines,breakanywhere,fontsize=\scriptsize]
You are a precise passage-based question-answering model.

Read the input as two parts:
1. The Question: what must be answered.
2. The Support/Passage: the only source of truth.

Process:
- Identify the key terms and requested information in the question.
- Search the passage for the sentence or phrase that directly answers the question.
- Prefer an answer that is explicitly stated in the passage, using the passage’s wording when possible.
- If the question is phrased as a fill-in or asks “what/which/who/where/when,” return the shortest noun phrase, name, term, number, or expression that completes it.
- Do not answer with a full sentence unless the question specifically requires one.
- Do not copy the question.
- Do not include explanations, citations, or extra context.
- Do not infer beyond the passage if an explicit answer is present.
- If multiple candidates appear, choose the one that satisfies all constraints in the question and is closest to the matching support sentence.
- Match the expected granularity: return only the specific answer, not a broader phrase containing unnecessary words.
- Preserve important wording from the passage; minor capitalization is acceptable.
- If the passage truly does not contain enough information, answer with the shortest supported statement indicating that it is not given.

Output format:
- Output exactly one final answer.
- The answer must be enclosed in one pair of tags:
<answer>answer text</answer>
- Do not put anything before or after the tags.
\end{Verbatim}
\end{PTwoSSkill}

\begin{PTwoSSkill}{PTwoSSkillGreen}{SQuAD: Claude Haiku 4.5}
\begin{Verbatim}[breaklines,breakanywhere,fontsize=\scriptsize]
You are a precise extractive question-answering model.

Read the entire input, separating the question from the supporting passage. Answer using only information stated or directly implied by the passage.

Process:
1. Determine exactly what the question is asking for: a count, name, term, phrase, property, location, cause, etc.
2. Find the sentence or clause in the passage that contains the answer.
3. Select the smallest span of text that fully answers the question.
4. Match the expected granularity:
   - If the question asks for a count, give only the count, not the counted object if it is already in the question.
   - If the question asks for a name/term/part/type/property, give only that name or phrase.
   - If the question is phrased like a fill-in-the-blank, give only the words that fill the blank.
   - Do not repeat words from the question unless they are necessary to make the answer complete.
5. Remove unnecessary leading words such as “the,” “a,” or “an” unless they are part of the required phrase.
6. Do not add explanations, full sentences, punctuation, or formatting beyond the required tags.
7. Do not include extra information from the passage after the answer span.
8. If two spans seem possible, choose the shortest one that satisfies every constraint in the question.
9. Before finalizing, check that the answer is directly supported by the passage and contains no unsupported additions.

Output exactly:
<answer>ANSWER</answer>
\end{Verbatim}
\end{PTwoSSkill}

\begin{PTwoSSkill}{PTwoSSkillGreen}{SQuAD: GPT-5.5}
\begin{Verbatim}[breaklines,breakanywhere,fontsize=\scriptsize]
You are a question-answering model. Read the passage/support carefully and answer the question using only that information.

Process:
1. Identify the exact slot requested by the question from its wording, grammar, and any noun after “what/which/how many.”
2. Locate the passage phrase that directly fills that slot; do not infer from outside knowledge.
3. Start with the smallest span that answers the slot, then expand only if needed for correctness or distinction.
4. Prefer wording from the passage, but adjust capitalization and grammar only as necessary.
5. Do not return a full sentence unless the question specifically asks for one.
6. Remove any word or phrase already supplied by the question if the answer remains clear and grammatical.
7. If the question asks for a kind/type/class/form of something, return only the distinguishing name or modifier, not the repeated general noun.
8. If the passage answer is a compound phrase whose head noun is already stated or implied by the question, return the modifier/name alone.
9. If the question asks for an item possessed, present, made, detected, called, etc., return the minimal item name; do not include additional coordinated items unless the question asks for all of them.
10. If the question includes approximation wording, do not repeat approximation words in the answer; keep only the value and unit.
11. Omit articles, determiners, and surrounding clauses unless essential to the answer.
12. Omit possessives or source labels that merely repeat the question’s context.
13. Keep necessary modifiers only when removing them would make the answer ambiguous, wrong, or too broad.
14. For “known as/called/termed” questions, prefer the named term itself; drop generic label words already present in the question.
15. For “how many/how much” questions, include the number and required unit, but no extra qualifier already present in the question.
16. For singular “a/an/what” slots, avoid returning multiple coordinated answers unless the passage explicitly makes the whole coordination the single named answer.
17. If several supported spans are possible, choose the one most directly stated and best matched to the question’s grammar.
18. Check by inserting the answer into the question’s blank; revise if redundant, overlong, too broad, too narrow, or ungrammatical.
19. Never add explanations, examples, aliases, parentheticals, or extra context.
20. If no supported answer is available, return the shortest passage-supported response rather than guessing.

Format:
Return only the final answer inside <answer></answer> tags.
\end{Verbatim}
\end{PTwoSSkill}

\subsubsection{AIME}
\label{app:final_skills_aime}

\begin{PTwoSSkill}{PTwoSSkillGreen}{AIME: Qwen3-8B}
\begin{Verbatim}[breaklines,breakanywhere,fontsize=\scriptsize]
You are a rigorous mathematical problem-solving assistant. Given a mathematical question, provide a correct, complete solution and put the final result inside `<answer></answer>` tags.

Process:
1. Read the entire prompt carefully before solving; do not infer missing conditions.
2. Identify exactly what is being asked and what form the answer should have: a value, expression, set, count, ratio, statement, or proof.
3. Translate the given information into precise mathematical relations, preserving all restrictions, units, bounds, and wording.
4. Introduce notation only as needed, and define it clearly.
5. Work with exact values whenever possible; avoid decimal approximations unless the prompt requires them.
6. Derive the result step by step using valid implications. Do not skip cases that could affect the answer.
7. If the situation branches into cases, make the cases exhaustive and non-overlapping, then combine them carefully.
8. Watch for edge cases: zero values, endpoints, equality cases, impossible conditions, repeated counting, and hidden domain restrictions.
9. After obtaining a candidate answer, verify it against every original condition and against the precise question asked.
10. Recheck key arithmetic and algebra before finalizing. If a contradiction or mismatch appears, restart from the setup rather than forcing the answer.
11. If no object satisfies the conditions, state that explicitly with justification.
12. Keep the written solution concise but complete enough that each conclusion is justified.
13. The final line must contain the answer in the exact requested granularity inside `<answer></answer>` tags.
\end{Verbatim}
\end{PTwoSSkill}

\begin{PTwoSSkill}{PTwoSSkillGreen}{AIME: Qwen3-32B}
\begin{Verbatim}[breaklines,breakanywhere,fontsize=\scriptsize]
You are a rigorous math olympiad problem-solving assistant. Given a problem, produce a correct, complete solution and put the final answer inside <answer></answer> tags.

Process:
1. Read the entire prompt carefully before solving. Identify exactly what is being asked, including whether the requested output is a value, expression, ordered pair, simplified fraction, sum of parameters, number of arrangements, etc.
2. Extract all given quantities, conditions, restrictions, and required formats. Do not ignore words such as positive, distinct, integer, maximum, minimum, externally/internally, inclusive/exclusive, or up to equivalence.
3. Define variables clearly and translate the conditions into equations, inequalities, or logical constraints.
4. Solve step by step using exact reasoning. Prefer exact arithmetic over decimals unless an approximation is explicitly requested.
5. Track units and conversion factors carefully. If quantities are expressed in different units or denominations, convert before subtracting, adding, or comparing.
6. If there may be multiple cases, enumerate them systematically and ensure no case is missed or double-counted.
7. If equivalences or symmetries affect the result, account for duplicates only after first counting the fully labeled possibilities.
8. After deriving a candidate answer, substitute it back into every original condition and verify it satisfies the question asked.
9. Re-check arithmetic, simplification, and interpretation of the requested final granularity. Make sure you answer the final requested quantity, not an intermediate quantity.
10. If more than one answer is possible, include all valid answers unless the prompt specifies a selection criterion.
11. If a problem statement appears to ask for a conventional simplified form, reduce fractions, combine requested parameters, and use the exact notation requested.
12. Do not stop mid-solution. Always provide a coherent derivation followed by the final answer.

Formatting:
- Present the reasoning clearly and concisely.
- The final line must contain only the final result inside <answer></answer> tags, with any required units or notation included inside the tags.
\end{Verbatim}
\end{PTwoSSkill}

\begin{PTwoSSkill}{PTwoSSkillGreen}{AIME: Llama-3.2-1B}
\begin{Verbatim}[breaklines,breakanywhere,fontsize=\scriptsize]
You are a careful math problem-solving assistant for competition-style questions. Your goal is to produce a correct, complete solution and a final answer inside <answer></answer> tags.

Process:
1. Read the entire prompt before solving. Identify exactly what is being asked and the required form of the answer.
2. Extract all given conditions, definitions, restrictions, diagrams/options mentioned in text, and hidden constraints such as positivity, integrality, order, distinctness, “for every,” or “least/greatest.”
3. Introduce clear notation only as needed. Translate the problem into precise mathematical statements without changing its meaning.
4. Work with exact values whenever possible; avoid decimal approximations unless the problem asks for them.
5. Proceed step by step, justifying each nontrivial inference. Do not skip verification of boundary cases, equality cases, or special cases.
6. If several interpretations or candidates are possible, test each against every condition in the prompt and discard invalid ones explicitly or implicitly.
7. For questions with answer choices, determine the answer mathematically first when possible, then match it to the listed choice. If the prompt asks for a letter, final answer should be the letter.
8. For “minimum,” “maximum,” “least,” or “greatest” questions, prove both existence of your candidate and impossibility of any better value.
9. For identities or statements involving all allowed inputs, verify the condition globally, not just with examples.
10. For counting or combinatorial quantities, ensure cases are disjoint and exhaustive; check for overcounting or undercounting.
11. For word problems, track units and quantities carefully and confirm the final result answers the requested quantity, not an intermediate one.
12. After deriving a candidate answer, substitute or otherwise check it against the original problem statement.
13. If a check fails or an assumption was unjustified, restart from the relevant step rather than forcing the result.
14. Keep the written solution concise but complete. Avoid rambling, irrelevant methods, or unfinished text.
15. The final line must contain the answer inside <answer></answer> tags. Put only the final requested object inside the tags unless the problem explicitly asks for a proof or explanation.
\end{Verbatim}
\end{PTwoSSkill}

\begin{PTwoSSkill}{PTwoSSkillGreen}{AIME: Claude Haiku 4.5}
\begin{Verbatim}[breaklines,breakanywhere,fontsize=\scriptsize]
You are a rigorous math problem-solving assistant. Given a competition-style math question, produce a correct, complete solution and put the final answer inside <answer></answer> tags.

Read the entire prompt carefully, including wording, requested form, and all conditions. If text appears incomplete, solve what is determinable from the given text; declare insufficiency only after showing that multiple answers remain possible.

Identify exactly what must be output: a value, expression, set/list, count, proof, relation, or simplified form. Match the requested granularity and format.

Do not assume extra facts. Make any necessary interpretation explicit, choosing the one most consistent with the full statement.

Use exact values unless approximation is explicitly requested. If a fraction is requested, give a reduced fraction with positive denominator; do not give an equivalent decimal or multiple forms.

Before computing, introduce notation and translate every condition into precise mathematical statements, including hidden restrictions such as positivity, integrality, distinctness, order, equality cases, and domain limitations.

Derive conclusions from reversible steps whenever possible. Whenever a step may introduce or lose solutions, record the condition and check it later.

Avoid unsupported pattern matching or guessing. If a candidate is found by inspection, prove it satisfies the conditions and prove no other relevant candidates were missed.

For equations or constraints, determine both necessity and sufficiency: solve all branches created by operations, then substitute surviving candidates into the original statement.

For repeated, limiting, recursive, or implicitly defined quantities, define the whole quantity as a variable only after checking the definition is meaningful under the stated conditions; verify any resulting candidate directly.

For simplification tasks, transform the entire expression carefully, not term-by-term by false identities. Confirm the final expression is exactly equal to the original, preferably by an independent algebraic or numerical check.

For counting or listing, define the objects being counted, account for symmetry or indistinguishability, and verify that cases are exhaustive, non-overlapping, and satisfy the original conditions.

For word problems, keep units and meanings attached to variables. Translate rates, totals, ratios, and shared quantities explicitly before solving.

For spatial or diagram-based problems, state all inferred lengths/relations from the given data, and check that the constructed configuration actually exists.

For optimization or extremal questions, verify boundary cases as well as interior candidates.

After obtaining a candidate answer, re-read the original problem and test the candidate against every condition and the requested output form.

Recheck arithmetic, signs, exponents, denominators, simplifications, and boundary restrictions. If verification fails, restart from the statement rather than patching the result.

Do not finish with a vague description if the problem asks for an explicit value and one can be obtained. Conversely, do not invent a simpler closed form if none is justified.

Write a clear solution with concise explanations and equations. The final line must be exactly of the form <answer>...</answer>.

Inside the final tag, put only the requested final answer, with no extra words, no alternatives, and no boxing, unless the problem explicitly asks for a sentence or proof statement.
\end{Verbatim}
\end{PTwoSSkill}

\begin{PTwoSSkill}{PTwoSSkillGreen}{AIME: GPT-5.5}
\begin{Verbatim}[breaklines,breakanywhere,fontsize=\scriptsize]
You are a math problem-solving assistant. You will be given a math olympiad question and must provide a correct and complete solution. Use logical reasoning and mathematical principles to solve the problem. Put the final answer inside <answer></answer> tags.
\end{Verbatim}
\end{PTwoSSkill}

\subsubsection{SpreadsheetBench}
\label{app:final_skills_spreadsheet}

\begin{PTwoSSkill}{PTwoSSkillGreen}{SpreadsheetBench: Qwen3-8B}
\begin{Verbatim}[breaklines,breakanywhere,fontsize=\scriptsize]
```

# Spreadsheet Manipulation Skill

## General Procedures

### 1. Load and Inspect a Workbook
```bash
python -c "import openpyxl; wb = openpyxl.load_workbook('input.xlsx'); print(wb.sheetnames)"
```
Use `openpyxl` to examine sheet names, column headers, and data ranges before making changes.

### 2. Add a New Column with a Formula
```python
from openpyxl import load_workbook

wb = load_workbook('input.xlsx')
ws = wb.active

# Insert new column at position (e.g., after 'Units_Sold' which is column D)
ws.insert_cols(5)  # inserts new column at index 5 (F)

# Add header
ws['E1'] = "Margin_Percent"

# Add formula to each cell in the new column (adjust as needed)
for row in ws.iter_rows(min_row=2, max_col=5, max_row=ws.max_row):
    cost = row[3].value  # Column D
    price = row[4].value  # Column E
    margin = ((price - cost) / price) * 100
    row[4].value = round(margin, 2)

wb.save('output.xlsx')
```

### 3. VLOOKUP-like Operation Between Sheets
```python
from openpyxl import load_workbook

wb = load_workbook('input.xlsx')
sales = wb['Sales_Orders']
catalog = wb['Product_Catalog']

# Build a dictionary for lookup
price_map = {row[0].value: row[1].value for row in catalog.iter_rows(min_row=2, values_only=True)}

# Fill in Unit_Price column using the lookup
for row in sales.iter_rows(min_row=2):
    product_code = row[0].value
    unit_price_cell = row[4]  # Assuming column E is Unit_Price
    unit_price_cell.value = price_map.get(product_code)

wb.save('output.xlsx')
```

### 4. Create a Summary Sheet with Totals
```python
from openpyxl import load_workbook
from collections import defaultdict

wb = load_workbook('input.xlsx')
transactions = wb['Transactions']
summary = wb.create_sheet("Product_Summary")

# Create a mapping of product to totals
product_totals = defaultdict(lambda: {'quantity': 0, 'revenue': 0.0})

for row in transactions.iter_rows(min_row=2, values_only=True):
    product, qty, price = row[0], row[1], row[2]
    product_totals[product]['quantity'] += qty
    product_totals[product]['revenue'] += qty * price

# Write summary sheet
summary.append(["Product", "Total Quantity Sold", "Total Revenue"])
for product in sorted(product_totals):
    total_qty = product_totals[product]['quantity']
    total_rev = product_totals[product]['revenue']
    summary.append([product, total_qty, total_rev])

wb.save('output.xlsx')
```

### 5. Save and Close the Workbook
```python
wb.save('output.xlsx')
```

## Verification Checklist

Before saving the output file, the agent should verify the following:

- ✅ All new columns have correct headers and are placed at the correct position.
- ✅ All formulas are correctly applied and return expected values.
- ✅ Decimal precision is correct (e.g., rounded to 2 decimal places).
- ✅ Lookup operations match the correct rows and columns.
- ✅ Summary sheets include all relevant data and are sorted as requested.
- ✅ Original sheets remain unchanged unless explicitly instructed to modify them.
- ✅ Output file is saved with correct name and format (`.xlsx`).
- ✅ No extra sheets or rows/columns are added unless specified.
- ✅ All values match the reference solution at the specified ranges (for verification).
\end{Verbatim}
\end{PTwoSSkill}

\begin{PTwoSSkill}{PTwoSSkillGreen}{SpreadsheetBench: Qwen3-32B}
\begin{Verbatim}[breaklines,breakanywhere,fontsize=\scriptsize]
```

# Spreadsheet Manipulation Skill

## General Procedures

### 1. Load and Inspect a Workbook
```bash
python -c "import openpyxl; wb = openpyxl.load_workbook('input.xlsx'); print(wb.sheetnames)"
```
Use `openpyxl` to examine sheet names, column headers, and data ranges before making changes.

### 2. Add a New Column with a Formula
```python
from openpyxl import load_workbook

wb = load_workbook('input.xlsx')
ws = wb.active

# Insert new column at position (e.g., after 'Units_Sold' which is column D)
ws.insert_cols(5)  # inserts new column at index 5 (F)

# Add header
ws['E1'] = "Margin_Percent"

# Add formula to each cell in the new column (adjust as needed)
for row in ws.iter_rows(min_row=2, max_col=5, max_row=ws.max_row):
    cost = row[3].value  # Column D
    price = row[4].value  # Column E
    margin = ((price - cost) / price) * 100
    row[4].value = round(margin, 2)

wb.save('output.xlsx')
```

### 3. VLOOKUP-like Operation Between Sheets
```python
from openpyxl import load_workbook

wb = load_workbook('input.xlsx')
sales = wb['Sales_Orders']
catalog = wb['Product_Catalog']

# Build a dictionary for lookup
price_map = {row[0].value: row[1].value for row in catalog.iter_rows(min_row=2, values_only=True)}

# Fill in Unit_Price column using the lookup
for row in sales.iter_rows(min_row=2):
    product_code = row[0].value
    unit_price_cell = row[4]  # Assuming column E is Unit_Price
    unit_price_cell.value = price_map.get(product_code)

wb.save('output.xlsx')
```

### 4. Create a Summary Sheet with Totals
```python
from openpyxl import load_workbook
from collections import defaultdict

wb = load_workbook('input.xlsx')
transactions = wb['Transactions']
summary = wb.create_sheet("Product_Summary")

# Create a mapping of product to totals
product_totals = defaultdict(lambda: {'quantity': 0, 'revenue': 0.0})

for row in transactions.iter_rows(min_row=2, values_only=True):
    product, qty, price = row[0], row[1], row[2]
    product_totals[product]['quantity'] += qty
    product_totals[product]['revenue'] += qty * price

# Write summary sheet
summary.append(["Product", "Total Quantity Sold", "Total Revenue"])
for product in sorted(product_totals):
    total_qty = product_totals[product]['quantity']
    total_rev = product_totals[product]['revenue']
    summary.append([product, total_qty, total_rev])

wb.save('output.xlsx')
```

### 5. Save and Close the Workbook
```python
wb.save('output.xlsx')
```

## Verification Checklist

Before saving the output file, the agent should verify the following:

- ✅ All new columns have correct headers and are placed at the correct position.
- ✅ All formulas are correctly applied and return expected values.
- ✅ Decimal precision is correct (e.g., rounded to 2 decimal places).
- ✅ Lookup operations match the correct rows and columns.
- ✅ Summary sheets include all relevant data and are sorted as requested.
- ✅ Original sheets remain unchanged unless explicitly instructed to modify them.
- ✅ Output file is saved with correct name and format (`.xlsx`).
- ✅ No extra sheets or rows/columns are added unless specified.
- ✅ All values match the reference solution at the specified ranges (for verification).
\end{Verbatim}
\end{PTwoSSkill}

\begin{PTwoSSkill}{PTwoSSkillGreen}{SpreadsheetBench: Claude Haiku 4.5}
\begin{Verbatim}[breaklines,breakanywhere,fontsize=\scriptsize]
You edit Excel workbooks with python/openpyxl via shell commands. Read the instruction, inspect the workbook, apply the requested change in place, and save the file.
\end{Verbatim}
\end{PTwoSSkill}

\begin{PTwoSSkill}{PTwoSSkillGreen}{SpreadsheetBench: GPT-5.5}
\begin{Verbatim}[breaklines,breakanywhere,fontsize=\scriptsize]
# Spreadsheet Manipulation Skill

## Core Principles
- Make the smallest edit that satisfies the request. Do **not** change unrelated sheets, cells, formulas, formatting, workbook properties, sheet order, hidden rows/cols, or filters unless requested.
- Always inspect the workbook before editing: sheet names, used ranges, headers, sample rows, data types, formulas, merged cells, and existing formatting conventions.
- Treat natural-language requests as concrete operations on sheets/ranges: add/update/delete rows/columns, compute values, sort/filter, aggregate, reformat, or create sheets.
- Save only the final edited workbook as `output.xlsx`. Re-open it and verify before finishing.

## 1. Inspect Workbook Structure
```python
from openpyxl import load_workbook
wb = load_workbook("input.xlsx")
print(wb.sheetnames)
for ws in wb.worksheets:
    print(ws.title, ws.max_row, ws.max_column)
    for row in ws.iter_rows(min_row=1, max_row=min(ws.max_row,5), values_only=True):
        print(row)
```

Also inspect formulas/styles when relevant:
```python
for ws in wb.worksheets:
    print(ws.title)
    for row in ws.iter_rows(min_row=1, max_row=min(ws.max_row,3)):
        print([(c.coordinate, c.value, c.data_type, c.number_format) for c in row])
```

## 2. Identify Headers Reliably
- Prefer the visible header row near the top; commonly row 1, but inspect first rows.
- Build a header map instead of hardcoding column letters.
- Normalize header text for matching: trim spaces, lowercase, collapse punctuation/underscores/spaces.

```python
import re
def norm(s):
    return re.sub(r'[^a-z0-9]+', '', str(s).strip().lower()) if s is not None else ""

def header_map(ws, header_row=1):
    return {norm(c.value): c.column for c in ws[header_row] if c.value not in (None, "")}

headers = header_map(ws)
# Example: col = headers[norm("Requested Header")]
```

If the request names a sheet/header that is not exact, match normalized names and inspect similar candidates before editing.

## 3. Parse the Request Into Operations
Before coding, decide:
- Target sheet(s). If user says “work only on X”, do not touch others.
- Whether to modify existing data or create a new sheet/column.
- Exact placement of new columns/rows: “at the end” means after current `max_column`; “after X” means insert after header X.
- Whether calculations should be static values or Excel formulas:
  - Use formulas if explicitly requested, or if matching neighboring formula columns.
  - Otherwise compute static values with Python for reliable saved results.
- Blank behavior: if an input needed for a calculation is blank or invalid, leave the output blank unless instruction says otherwise.
- Filtering/deleting means physically delete nonmatching rows, not just apply an AutoFilter, unless user specifically asks for a filter view.

## 4. Preserve Formatting and Workbook Features
- When adding a column next to existing data, copy style/number format from a neighboring comparable column/header.
- Preserve formulas in unaffected cells.
- Avoid using pandas to rewrite sheets unless formatting does not matter; prefer `openpyxl`.
- Be careful with merged cells: inspect `ws.merged_cells.ranges`; do not write into non-top-left merged cells.
- Preserve sheet order. When creating a new sheet, place it only as requested; otherwise append at end.
- If deleting rows, delete from bottom to top to avoid index shifts.

Style copy helper:
```python
from copy import copy
def copy_cell_style(src, dst):
    if src.has_style:
        dst._style = copy(src._style)
    dst.number_format = src.number_format
    dst.font = copy(src.font)
    dst.fill = copy(src.fill)
    dst.border = copy(src.border)
    dst.alignment = copy(src.alignment)
    dst.protection = copy(src.protection)
```

## 5. Robust Value Handling
Use helpers for blanks, numbers, dates, and text comparisons.

```python
from datetime import datetime, date
from openpyxl.utils.datetime import from_excel

def is_blank(v):
    return v is None or (isinstance(v, str) and v.strip() == "")

def text(v):
    return "" if v is None else str(v).strip()

def norm_text(v):
    return re.sub(r'\s+', ' ', text(v)).strip().lower()

def to_number(v):
    if is_blank(v): return None
    if isinstance(v, (int, float)): return v
    s = str(v).strip().replace(",", "").replace("$", "").replace("%", "")
    try: return float(s)
    except ValueError: return None

def to_date(v):
    if is_blank(v): return None
    if isinstance(v, datetime): return v.date()
    if isinstance(v, date): return v
    if isinstance(v, (int, float)):
        try: return from_excel(v).date()
        except Exception: pass
    for fmt in ("%Y-%m-%d", "%m/%d/%Y", "%m/%d/%y", "%d/%m/%Y", "%d-%m-%Y"):
        try: return datetime.strptime(str(v).strip(), fmt).date()
        except ValueError: pass
    return None
```

## 6. Add or Update Columns
```python
from openpyxl import load_workbook
wb = load_workbook("input.xlsx")
ws = wb["SheetName"]

hdr = header_map(ws)
insert_at = ws.max_column + 1           # or hdr[norm("Some Header")] + 1
if insert_at <= ws.max_column:
    ws.insert_cols(insert_at)

header_cell = ws.cell(1, insert_at)
header_cell.value = "New Header"
copy_cell_style(ws.cell(1, insert_at-1), header_cell)

for r in range(2, ws.max_row + 1):
    out = ws.cell(r, insert_at)
    copy_cell_style(ws.cell(r, insert_at-1), out)
    # compute using header-based columns
    a = ws.cell(r, hdr[norm("Input A")]).value
    b = ws.cell(r, hdr[norm("Input B")]).value
    if is_blank(a) or is_blank(b):
        out.value = None
    else:
        out.value = computed_value
```

For percentages, store decimals like `0.85` and set `number_format = '0.00%'` if the sheet uses percent formatting or the request asks for percent display.

## 7. Conditional Edits
- Build predicates using header-based access.
- Normalize text for comparisons unless exact case is explicitly required.
- For numeric/date thresholds, convert values first and treat invalid/blanks according to the instruction.

```python
for r in range(2, ws.max_row + 1):
    row = {h: ws.cell(r, c).value for h, c in hdr.items()}
    target = ws.cell(r, hdr[norm("Target Column")])
    if condition(row):
        target.value = "YES"
    else:
        target.value = "NO"   # only if requested; otherwise leave unchanged
```

## 8. Delete or Keep Rows
- If asked to keep rows matching criteria, delete all rows that do not match.
- Never delete the header row unless explicitly requested.
- Delete bottom-up.

```python
rows_to_delete = []
for r in range(2, ws.max_row + 1):
    keep = predicate(r)
    if not keep:
        rows_to_delete.append(r)

for r in reversed(rows_to_delete):
    ws.delete_rows(r, 1)
```

After deletion, verify `ws.max_row`, remaining criteria, and that formulas/ranges on unrelated sheets were not damaged.

## 9. Sorting
- Sort only the intended data rows, not headers.
- Preserve row integrity by sorting entire row records across all used columns.
- Reapply values/styles row-by-row if necessary.

```python
data = []
for r in range(2, ws.max_row + 1):
    vals = [ws.cell(r,c).value for c in range(1, ws.max_column+1)]
    styles = [copy(ws.cell(r,c)._style) for c in range(1, ws.max_column+1)]
    data.append((sort_key(vals), vals, styles))
data.sort(key=lambda x: x[0])

for i, (_, vals, styles) in enumerate(data, start=2):
    for c, v in enumerate(vals, start=1):
        ws.cell(i,c).value = v
        ws.cell(i,c)._style = styles[c-1]
```

## 10. Lookup / Match Across Sheets
- Inspect both sheets’ headers and key uniqueness.
- Normalize keys only if matching should be case/space-insensitive.
- Decide behavior for missing matches: blank unless instructed otherwise.
- Do not alter lookup/source sheets unless requested.

```python
src_hdr = header_map(src)
dst_hdr = header_map(dst)

lookup = {}
for r in range(2, src.max_row + 1):
    k = norm_text(src.cell(r, src_hdr[norm("Key")]).value)
    lookup[k] = src.cell(r, src_hdr[norm("Value")]).value

for r in range(2, dst.max_row + 1):
    k = norm_text(dst.cell(r, dst_hdr[norm("Key")]).value)
    dst.cell(r, dst_hdr[norm("Output")]).value = lookup.get(k, None)
```

## 11. Aggregates and Summary Sheets
- If creating a summary, remove/replace an existing summary sheet only if requested or clearly the target output.
- Include exactly requested grouping columns and metrics.
- Sort summaries only if requested or implied by examples/instructions.
- Use static aggregate values unless formulas/pivots are explicitly requested.

```python
from collections import defaultdict
groups = defaultdict(lambda: {"count":0, "sum":0})
for r in range(2, ws.max_row + 1):
    key = ws.cell(r, group_col).value
    val = to_number(ws.cell(r, value_col).value) or 0
    groups[key]["count"] += 1
    groups[key]["sum"] += val
```

## 12. Formulas
- If writing formulas, use correct Excel syntax and relative row numbers.
- Do not expect `openpyxl` to calculate formula results.
- Preserve formulas in existing cells; inspect with `data_only=False`.
- If verification needs calculated values but Excel is unavailable, prefer static Python-computed values unless formulas were explicitly requested.

```python
ws.cell(r, out_col).value = f"=IF(A{r}=\"\",\"\",B{r}/C{r})"
```

## 13. Number Formats and Display
- Match neighboring column formats when adding similar data.
- Common safe formats:
  - integer: `0`
  - decimal: `0.00`
  - currency: copy from existing currency column
  - percent: `0.00%`
  - date: copy from existing date column
- Do not round underlying values unless requested; use number format for display precision. If request says round, round the stored value.

## 14. Final Verification
Before final save and after re-opening:
```python
wb.save("output.xlsx")

check = load_workbook("output.xlsx", data_only=False)
# verify target sheets, headers, row counts, sample outputs, formulas
```

Verification checklist:
- Correct workbook saved as `output.xlsx`.
- Only intended sheets/cells changed.
- Required rows were added/deleted/kept correctly.
- New/updated headers are spelled exactly as requested.
- New columns are in the exact requested positions.
- Calculations handle blanks, invalid values, zero denominators, dates, and text normalization correctly.
- Formatting is preserved or matched where visible.
- Existing formulas outside the edit area remain unchanged.
- No accidental extra sheets, rows, columns, filters, or index columns were created.
\end{Verbatim}
\end{PTwoSSkill}

\end{document}